\documentclass[journal,final]{IEEEtran}
\usepackage{cite,graphicx}
\usepackage{amssymb,amsmath,amsthm,multicol,cite}
\usepackage{algorithm,algorithmic,multirow,color,mysymbol}
\usepackage[english]{babel}
\usepackage{mysymbol}

\newtheorem{thm}{Theorem}[section]

\newtheorem{lem}[thm]{Lemma}
\newtheorem{prop}[thm]{Proposition}
\newtheorem{problem}{Problem}

\def\hbbx{{\hat{\ensuremath{\mathbf x}} }}
\def\hbbz{{\hat{\ensuremath{\mathbf z}} }}
\def\cbbp{{\check{\ensuremath{\mathbf p}} }}
\def\cbbx{{\check{\ensuremath{\mathbf x}} }}

\allowdisplaybreaks

\begin{document}
\title{Controlling a Robotic Stereo Camera Under Image Quantization Noise}
\author{
Charles~Freundlich,~\IEEEmembership{Student~Member,~IEEE,}
Yan~Zhang,~\IEEEmembership{Student~Member,~IEEE,}
Alex~Zihao~Zhu,~\IEEEmembership{Student~Member,~IEEE,}
Philippos~Mordohai,~\IEEEmembership{Member,~IEEE,} and 
Michael~M.~Zavlanos,~\IEEEmembership{Member,~IEEE}
\thanks{Charles Freundlich, Yan Zhang, and Michael M. Zavlanos are with the Dept. of Mechanical Engineering and Materials Science, Duke University, Durham, NC 27708, USA {\tt\footnotesize \{charles.freundlich, yz227, michael.zavlanos\}@duke.edu}.
Philippos Mordohai is with the Dept. of Computer Science, Stevens Institute of Technology, Hoboken, NJ 07030, USA {\tt\footnotesize mordohai@cs.stevens.edu}.
Alex Zihao Zhu is with the Dept. of Computer and Information Science, University of Pennsylvania, Philadelphia, PA, USA {\tt\footnotesize alexzhu@seas.upenn.edu}. This work is supported in part by the National Science Foundation under awards No. CNS-1302284, IIS-1217797, and IIS-1637761.  Preliminary versions of this work can be found in \cite{freundlich13icra,freundlich13cdc}.
} }


\maketitle
\begin{abstract}
In this paper, we address the problem of controlling a mobile stereo camera under image quantization noise.
Assuming that a pair of images of a set of targets is available, the camera moves through a sequence of Next-Best-Views (NBVs), i.e., a sequence of views that minimize the trace of the targets' cumulative state covariance, constructed using a realistic model of the stereo rig that captures image quantization noise and a Kalman Filter (KF) that fuses the observation history with new information.
The proposed algorithm decomposes control into two stages: first the NBV is computed in the camera relative coordinates, and then the camera moves to realize this view in the fixed global coordinate frame.
This decomposition allows the camera to drive to a new pose that effectively realizes the NBV in camera coordinates while satisfying Field-of-View constraints in global coordinates, a task that is particularly challenging using complex sensing models.
We provide simulations and real experiments that illustrate the ability of the proposed mobile camera system to accurately localize sets of targets.
We also propose a novel data-driven technique to characterize unmodeled uncertainty, such as calibration errors, at the pixel level and show that this method ensures stability of the KF.
\end{abstract}


\section{Introduction}
\label{sec:intro}
Active robotic sensors are rapidly gaining viability in environmental, defense, and commercial applications.
As a result, developing information-driven sensor strategies has been the focus of intense and growing research in artificial intelligence, control theory, and signal processing. 
We focus on stereoscopic camera rigs, that is, two rigidly connected cameras in a pair.
Specifically, we address the problem of determining the trajectory of a mobile robotic sensor equipped with a stereo camera rig so that it localizes a collection of possibly mobile targets as accurately as possible under image quantization noise.

The advantage of binocular vision, compared to the use of monocular camera systems, is that it provides both depth and bearing measurements of a target from a pair of simultaneous images. 
Assuming that noise is dominated by quantization of pixel coordinates \cite{blostein87,matthies87,chang94} we use the measurement Jacobian to propagate the error from the pixel coordinates to the target coordinates relative to the stereo rig.
In particular, we approximate the pixel error as Gaussian and propagate the noise to the target locations, giving rise to fully correlated second order error statistics, or measurement error covariance matrices, which capture target location uncertainty.
The resulting second order statistic is an accurate representation of not only the eigenvalues but also the eigenvectors of the measurement error covariance matrices, which play a critical role in active sensing as they determine viewing directions from where localization uncertainty can be further decreased.

Assuming that a pair of images of the targets is available, in this paper, we iteratively move the stereo rig through a sequence of configurations that minimize the trace of the targets' cumulative covariance.
This cumulative covariance is constructed using a Kalman Filter (KF) that fuses the observation history with the predicted instantaneous measurement covariance obtained from the proposed stereoscopic sensor model. 
Differentiating this objective with respect to the new instantaneous measurement in the relative camera frame provides the Next Best View (NBV), i.e., the new relative distance and direction from where a new measurement should be obtained. 
Then, the stereo rig moves to realize this NBV using a gradient descent approach in the joint space of camera rotations and translations.  
Once the NBV is realized in the global frame, the camera takes a new pair of images of the targets that are fused with the history using the KF to update the prior cumulative covariance of the targets, and the process repeats with the determination of a new NBV. 
The sequence of observations and resulting NBVs, generated by the proposed iterative scheme, constitutes a switching signal in the continuous motion control.
During motion, appropriate barrier potentials prevent targets from exiting the camera's geometric Field-of-View (FoV).
As we illustrate in computer simulations and real-world experiments, the resulting sensor trajectory balances between reducing range and diversifying viewpoints, a result of the eigenvector information contained in the posterior error covariances. 
This behavior of our controller is notable when compared to existing sensor guidance approaches that adopt approximations to the error covariance that are not direct functions of the stereo rig calibration and the pixel observations themselves
\cite{le1996optimization,passerieux1998optimal,logothetis1997information,stroupe05,zhou_roumeliotis11,olfati2007distributed,chung06}.

\subsection{Related Work}\label{sec_related}
Our work is relevant to a growing body of literature that addresses control for one or several mobile sensors for the purpose of target localization or tracking
\cite{fox00,roumeliotis02,spletzer03,stroupe05,chung06,yang07,morbidi2013active,zhou_roumeliotis11}.
These methods use sensor models that are based only on range and viewing angle.
These models, if used for stereo triangulation, can not accurately capture the covariance among errors in measurement coordinates, nor can they capture dependence on range and viewing angle.
It is also common to ignore directional field of view constraints by assuming omnidirectional sensing.
In this paper, we derive the covariance specifically for triangulation with a calibrated stereo rig.
The derived measurement covariance, when fused with a prior distribution, provides our controller with critical directional information that enables the mobile robot to find the NBV, defined as the vantage point from where new information will reduce the posterior variance of the targets' distribution by the maximum amount.

Recent work \cite{ponda09,Adurthi13,ding2012coordinated} brought about by developments in fixed-wing UAV control, addresses the autonomous visual tracking and localization problem using optimal control over information-based objectives using monocular vision.
\cite{ponda09} define an objective function based on the trace of the covariance matrix of the target location and determine the next best view by a numerical gradient descent scheme. 
\cite{ding2012coordinated} also minimize the trace of the fused covariance by guiding multiple non-holonomic Unmanned Aerial Vehicles (UAVs) that use the Dubins car model.
\cite{Adurthi13} use receding horizon control to maximize mutual information.
Their method avoids replanning at every step by doing so only when the Kullback-Leibler divergence between the most recent target location probability density function (pdf) and the pdf that was used in the planning phase differ by a user-specified threshold.
The Dynamic Programming (DP) approaches of \cite{ponda09,Adurthi13,ding2012coordinated} have complexity that grows exponentially in the horizon length.
In this paper, our proposed analytical and closed-form expression for the gradient guides image collection in all position and orientation directions continuously.
Although it does not plan multiple steps into the future, our controller is adaptive due to its feedback nature; each decision predicts new sensor locations from where new measurements optimize the estimated target locations based on the full, fused observation history.

As far back as \cite{bajcsy88}, computer vision researchers have recognized that sensing decisions should be based on an exact sensor model, and that robotic vision, like human vision, can benefit from mobility.
Relevant prior work on active vision controls the image collection process for digital cameras through a discretized pose space by optimizing a scalar function of the covariance of feature-points on an object that is to be reconstructed.
Specifically, \cite{trummer10} focus on the maximum eigenvalue of the posterior covariance, \cite{Wenhardt06} the entropy, and \cite{dunn_iros09} the expected quality of the next view.
While these works do obtain uncertainty estimates that depend on factors such as viewing distance and camera resolution, which improves accuracy in 3D reconstruction, they do not operate continuously in 3D or consider dynamic environments.
\cite{morbidi2013active} optimize an objective function that depends on the covariance matrix of the KF, rather than the measurement error covariance matrix.  
The authors derive upper and lower bounds on the covariance matrix at steady-state and validate their method in simulation.
Stereoscopic vision sensors in continuous pose space are employed by
\cite{Shade10}, similar to the work proposed here.
However, this work \cite{Shade10} is concerned with exploration of an indoor environment and not with refining localization estimates.

When the target configurations can be collectively modeled by a coarse mesh in space, the NBV problem becomes similar to active inspection.
Several researchers have addressed this problem using approximate dynamic programming, by formulating it as coverage path planning.
\cite{galceran2013} provide an overview of coverage problems in mobile robotics, where the goal is to plan sensor paths that ``see'' every point on the surface mesh.
Similarly, \cite{wang2007} propose solving the traveling view path planning problem using approximate integer programming on a network flow model.
\cite{papadopoulos2013} enforce differential constraints for this problem.
\cite{hollinger13} invoke adaptive submodularity, which argues that greedy approaches to measurement acquisition may outperform dynamic programming approaches that do not replan as measurements are acquired.
While relevant to this work from a sensor planning perspective, active inspection methods do not address target localization. 
Moreover, dynamic and integer programming methods tend to be computationally expensive, especially for high dimensional spaces as those resulting from the presence of mobile targets. 
In this paper, we assume that there are no occlusions and, therefore, coverage (or detection) can be obtained if FoV constraints are met.
Moreover, we assume no correspondence errors between images. 
These assumptions allow us to develop a control systems approach to the target localization problem that is based on computationally efficient, analytic, expressions for the camera motion and image collection process, as well as on precise sensor models that can result in more accurate localization.

We note briefly that this paper is based on preliminary results contained in our prior publications \cite{freundlich13icra,freundlich13cdc}.
These early works used simplified versions of the noise model and global controller and lacked experimental validation.

\subsection{Contributions}
Our proposed control decomposition and resulting hybrid scheme possess a number of advantages compared to other methods that control  directly the full non-linear system or resort to dynamic programming for nonmyopic planning.
While these methods can have their own benefits, they also suffer from drawbacks. 
In particular, controlling directly the full non-linear system can be subject to multiple local minima that might be difficult to handle. 
On the other hand, dynamic programming formulations suffer from computational complexity due the size of the resulting state-spaces and often resort to abstract sensor models to help reduce complexity
\cite{logothetis1997information,stroupe05,singh2007simulation,logothetis1998comparison,frew2003trajectory,Adurthi13,ding2012coordinated}.
Additionally, these approaches use discrete methods, e.g., the exhaustive search of \cite{frew2003trajectory} and the gradient approximations of \cite{singh2007simulation}, to achieve the desired control task. 
Instead, decomposing control in the global and relative frames allows us to consider separately high-level planning, defined by the image collection/sensing process, and low-level planning, i.e., motion control of the camera. 
An advantage of this decomposition is that, given an NBV in the relative frame, there are infinite ways that the camera can realize this NBV in the global frame.
This provides choices to the motion controller that otherwise could be subject to local stationary points due to the nonlinear coupling between sensing and planning.
We provide a stability proof of the motion controller, while extensive computer simulations and experimental results have shown that even when FoV constraints are considered, local minima are not an issue and can be avoided by simple tuning of a gain parameter. 
The control decomposition also allows us to introduce Field-of-View (FoV) constraints that have not been previously used in the NBV context due to the complexity of their implementations. 
Most authors have used omnidirectional sensor models to circumvent these difficulties. 
The FoV constraints naturally enter the motion controllers when control is decomposed in the global and relative frames. 
Finally, to the best of our knowledge, the approaches by
\cite{stroupe05, zhou_roumeliotis11,olfati2007distributed,Adurthi13,ding2012coordinated}
rely on having large numbers of sensors, e.g., 20 or 60, and consider a single target, while our method enables one sensor to track multiple targets as long as they satisfy the FoV constraints.

In summary, the contribution of this work is that we address the multi-target, single-sensor problem employing the most realistic sensor model among continuous-space approaches in the literature that rely on the gradient of an optimality metric of the error covariance for planning.
Additionally, to the best of our knowledge, this work is the first to include FoV constraints within the NBV setting.
We also model image quantization noise directly.
This allows us to accurately model the second order error statistics of the target location uncertainty based on the actual pixel error distribution.
While other sources of error, such as association (matching) errors or occlusion, contribute to target localization error, a simultaneous and exact treatment of all error sources for the purposes of active sensing is an open problem.
In this work we have addressed an essential contributor, that is, quantization.
We have also proposed a novel data-driven technique to account for unmodeled uncertainty, such as system calibration errors, that is necessary to transition the proposed theoretical results to practice.
In particular, for long range stereo vision, calibration errors are unique to the particular stereo rig used.
They can cause severe bias and ill conditioned covariance matrices that may be completely different from one stereo rig to another.
As our method heavily relies on the KF, ensuring that measurements are unbiased and that we have a reliable estimate for their covariance is crucial to both convergence of the estimator and for generating sensible closed-loop robot trajectories.
Our proposed data-driven technique corrects measurements at the pixel level and empirically calculates their predictive error covariances.
Specifically, using a sufficiently large training set of stereo image pairs, we determine the empirical error covariance, which we propagate to world coordinates and use for both path planning and state estimation.
To the best of our knowledge, our data-driven approach to estimating the error statistics in the pixel coordinates is novel.
Most relevant literature in stereo vision assumes arbitrarily large such statistics, so that the estimation process is stable.
In practice, we found this step to be crucial for accurate triangulation and fusion of multiple measurements, even in our controlled lab environment.

We note that this paper is based on preliminary results that can be found in \cite{freundlich13icra,freundlich13cdc}.
The main differences between our preliminary work in \cite{freundlich13icra,freundlich13cdc} and the work proposed here are the following.
First, here we present thorough experimental results that validate our approach; the first of their kind for stereo vision.
Second, the controller proposed in \cite{freundlich13icra,freundlich13cdc} realizes a NBV that does not place the targets on the positive $z$-axis (viewing direction) of the stereo rig.
Looking straight at the targets results in more accurate observations.
The controller proposed here has this property.
As a result, the correctness proofs in this paper are different compared to \cite{freundlich13icra,freundlich13cdc}.
Finally, the noise model in this paper is based on an empirical model of quantization noise in stereo vision, as opposed to constant pixel noise covariance in \cite{freundlich13icra,freundlich13cdc}.

The paper is organized as follows.
Section~\ref{sec_problem} formulates the visual target tracking problem.  
Section~\ref{sec:potdes} discusses the NBV in the camera-relative coordinate system.  
Section~\ref{sec:realize} presents the gradient flow in the global coordinate frame.  
Section~\ref{sec_simulations} illustrates the proposed integrated hybrid system via computer simulations for static and mobile target localization and discusses ways to integrate FoV constraints in the proposed controllers. 
Section~\ref{sec:exp} gives experimental validation of our claims and describes the data-driven noise modeling strategy.
Section~\ref{sec_conclusions} concludes the work.

\section{System Model}\label{sec_problem}
Consider a group of $n$ mobile targets, indexed by $i \in \ccalN = \{1 \dots n \}$, with initially unknown positions $\bbx_i$.  Consider also a mobile stereo camera located at $\bbr(t) \in \reals^3$ and with orientation $R(t) \in SO(3)$, where $SO(3)$ denotes the special orthogonal group of dimension three, with respect to a fixed global reference frame at time $t\geq 0$.  A coordinate frame anchored to the stereo camera, hereafter referred to as the relative coordinates, is oriented such that, without loss of generality, the $x$-axis joins the centers of two monocular cameras and the positive $z$-axis measures range.  We denote the two cameras by Left (L) and Right (R).  The (L) and (R) camera centers are thus located at $(-b/2,0,0)$ and $(b/2,0,0)$ in the relative coordinates, where $b$ denotes the stereo baseline (see Fig.~\ref{fig:y_cord}).

\begin{figure}[t]
 \centering
 \includegraphics[width=7cm]{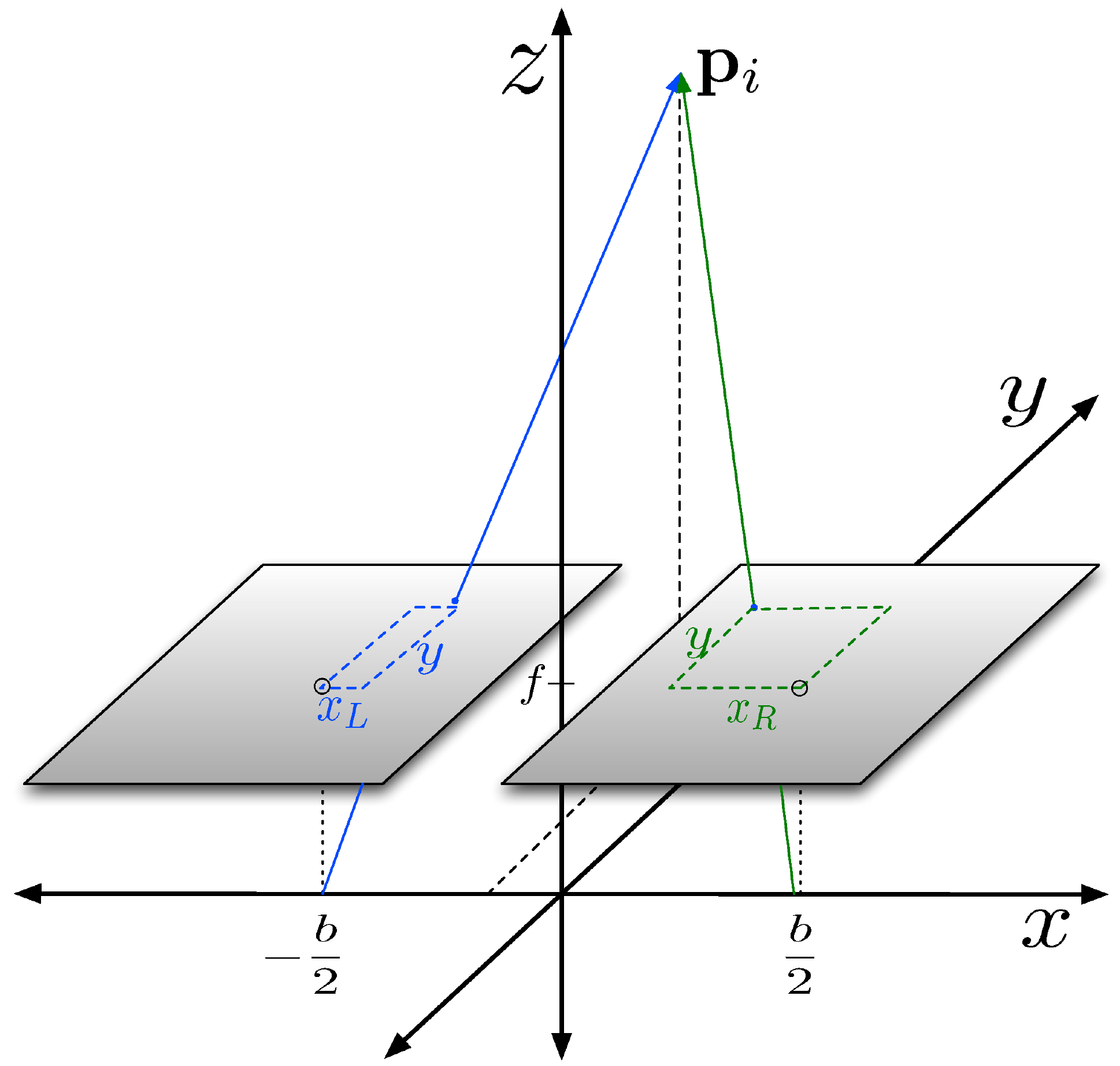}
\caption{Stereo geometry in 3D.  Two rays from the camera centers to a target located at $\bbp_i$ creates a pair of image coordinates, $(x_L,y)$ and $(x_R,y)$.}
 \label{fig:y_cord}
 \end{figure}

The position of target $i$ with respect to the relative camera frame can be expressed as
\begin{equation} \label{eq:bbp}
\bbp_i\triangleq \bbp(x_{Li},x_{Ri},y_i) = \frac{b}{x_{Li} - x_{Ri}} \begin{bmatrix}\frac{1}{2} (x_{Li}+x_{Ri}) \\y_i \\f \end{bmatrix},
\end{equation}
where $f$ denotes the focal length of the camera lens, measured in pixels, and $x_{Li}$, $x_{Ri}$, and $y_i$ denote the pixel coordinates of target $i$ on the left and right camera images, as in Fig.~\ref{fig:y_cord}, where we note that $y_i$ is equal in the left and right image by the epipolar constraint.
Given the orientation and position of the mobile camera, it is useful to consider the location of target $i$ in global coordinates
\begin{equation}\label{eq:hatx}
\bbx_{i}\triangleq R(t) {\bbp}_{i}+ \bbr(t).
\end{equation}

In practice, we can only observe quantized versions of the image coordinate tuples $(x_{Li},x_{Ri},y_{i})$ once they are rounded to the nearest pixel centers, which we hereafter denote by $\chkx_{Li}$, $\chkx_{Ri}$, and $\chky_{i}$.
In view of \eqref{eq:bbp}, the corrupted observation $(\chkx_{Li},\chkx_{Ri},\chky_{i})$ carries its quantization error into the observed coordinates $\bbp_i$ of target $i$, causing non-Gaussian error distributions \cite{blostein87,chang94}.  
For convenience, we follow \cite{matthies87,foerstner05} and approximate the quantized error in the pixels as Gaussian to allow uncertainty propagation from image to world coordinates.
The noise propagation takes place under a linearization of the measurement equation, so that the localization error of the target in the relative camera frame will also be Gaussian with mean $\cbbp_i=\bbp(\chkx_{Li},\chkx_{Ri},\chky_{i})$.
It follows from \eqref{eq:hatx} that the global location estimate $\cbbx_i$ is also subject to Gaussian noise.

Targets may be mobile, so we denote the ground truth full state of target $i$ by
$\bbz_i = [\bbx_i^\top  \; \dot{\bbx}_i^\top  \; \ddot{\bbx}_i^\top ]^\top \!\!$.  Then, $\bbx_i$ and $\bbz_i$ are related by $\bbx_i = H \bbz_i$, where $H = [1 \; 0 \; 0] \otimes I_3$, where $\otimes$ represents the Kronecker product.  Thus, we can think of $\cbbx_{i}$ as a noisy copy of the zero-th order terms of $\bbz_i$,
\begin{equation}\label{eq:lin_obs}
\cbbx_{i} = H \bbz_{i} + \bbv_{i},
\end{equation}
where $\bbv_{i}$ is a white noise vector.  We hereafter denote the covariance of $\bbv_{i}$ by $\Sigma_i\in \mbS^3_+$, where $\mbS^3_+$ denotes the set of $3\times3$ symmetric positive definite matrices.  In Section~\ref{sec:potdes}, we discuss an explicit form of $\Sigma_i$ that depends on the measurement itself.

\subsection{Kalman Filtering to Fuse the Target Observations}
Assume that the stereo camera has made a sequence of observations of the targets.  Introduce an index $k\geq 0$ associated with every observation to obtain $\cbbx_{i,k}$ and associated covariances $\Sigma_{i,k}$ from \eqref{eq:lin_obs}.  Our goal is to create accurate state information for a group of targets based on a sequence of such observations.  For this, we use a Kalman filter (KF), which is an efficient filter that can incorporate a sequence of noisy measurements within a system model to create accurate state estimates. 

We model the continuous time evolution of target $i$'s motion with the discrete time linear equation
\begin{equation}\label{dis_time_eq}
{\bbz}_{i,k} = \Phi {\bbz}_{i,k-1}.
\end{equation}
In \eqref{dis_time_eq}, $\Phi$ is the state transition matrix, which is unknown to the observer.
Adaptive procedures for determining $\Phi$ are well studied in the literature on mobile target tracking
\cite{Singer70,Li03_part1}.
Zero velocity and constant acceleration models of the target trajectory, which we discuss in Section~\ref{sec_simulations}, are modeled over a short time interval $Dt$ by
\begin{subequations}
 \label{eq:genmod}
\begin{align}
\Phi_{\dot{\bbx}_i={\bf 0}} = \begin{bmatrix}	 1 		&  	0		&	0\\
							0 	&	0		&	0\\
							0 	&	0 		&	0\end{bmatrix}\otimes I_3\; \text{and} \\
\Phi_{\dddot{\bbx}_i={\bf 0}} = \begin{bmatrix}	 1 		&  	Dt 		&	\frac{Dt^2}{2}  \\
								0 	&	1		&	Dt \\
								0 	&	0 	&	1 \end{bmatrix}\otimes I_3,
\end{align}
\end{subequations}
The KF recursively creates state estimates, which we denote by $\hat{\bbz}_i$, and their associated covariances, which we denote by $U_i$.   In particular, given prior estimates $\hat{\bbz}_{i,k-1 | k-1}$ and $U_{i,k-1 | k-1}$, we update the state estimates and fuse the covariance matrices according to the following KF:
\begin{subequations} \label{eq:KFeqs}
\begin{align}
\hat{\bbz}_{i,k | k-1}& = \Phi \hat{\bbz}_{i,k-1 | k-1} \label{eq:zpred}\\
U_{i,k| k-1}  &= \Phi U_{i,k-1| k-1} \Phi^\top +W \\
K_k &= U_{i,k| k-1} H^\top \left[H U_{i,k| k-1}H^\top +\Sigma_{i,k}\right]^{-1}\!\!\!\! \label{eq:Kgain} \\
\hat{\bbz}_{i,k | k} &= \hat{\bbz}_{i,k | k-1} + K_k\left[\cbbx_{i,k}-H \hat{\bbz}_{i,k | k-1}\right]\\
U_{i,k| k}  &= U_{i,k| k-1} - K_k HU_{i,k| k-1}   \label{eq:covfinal}
\end{align}
\end{subequations}
where $W$ is the process noise covariance matrix and accounts for the approximate nature of $\Phi$.  
\cite{Singer70} gives a closed form for this matrix,
\begin{equation}\label{eq:processnoise}
W =  \begin{bmatrix} 	Dt^5/20	&	Dt^4/8		&	Dt^3/6		\\
							Dt^4/8		&	Dt^3/3		&	Dt^2/2		\\
							Dt^3/6		&	Dt^2/2		&	Dt		
\end{bmatrix}\otimes I_3.
\end{equation}
From equation \eqref{eq:covfinal} and the results of \cite{welch01}, a closed form expression for the fused covariance estimate follows in the form of a Lemma.
\begin{lem} \label{lem:kfchange}
Let $U_{i,k|k-1}$ denote the predicted covariance of all prior observations and $\Sigma_{i,k}$ denote the covariance of the most recent measurement.  Then, the location estimate of target $i$,  $H \hat{\bbz}_{i,k | k}$, has a covariance matrix, which we hereafter denote by $\Xi_{i,k}$, given by
\begin{equation}\label{eq:Xidef}
\Xi_{i,k} \triangleq HU_{i,k|k}H^\top  = \left[\left(HU_{i,k|k-1}H^\top \right)^{-1}+\Sigma_{i,k}^{-1}\right]^{-1}\!\!\!\!.
\end{equation}
\end{lem}
\begin{proof}From the definition of $\Xi_{i,k}$ we have that,
\begin{align*}
U_{i,k| k}  &= U_{i,k| k-1} - K_k HU_{i,k| k-1} \\
HU_{i,k| k}  H^\top &=H \left(U_{i,k| k-1} - K_kHU_{i,k| k-1}\right)H^\top =\Xi_{i,k}.
\end{align*}
To simplify the analysis, let $U=U_{i,k| k-1}$ and $\Sigma = \Sigma_{i,k}$.  Substituting the Kalman gain $K_k$ from \eqref{eq:Kgain}, we have
\begin{align*}
\Xi_{i,k} &=HUH^\top   \Big[ I - \big(HUH^\top + \Sigma\big)^{-1} HUH^\top  \Big] \\
&=HUH^\top   \Big[ \left(HUH^\top   + \Sigma\right)^{-1}  \left(HUH^\top   + \Sigma\right) -  \\
&\hspace{3cm}  \left(HUH^\top   + \Sigma\right)^{-1} HUH^\top  \Big]\\
&= HUH^\top \left(HUH^\top   + \Sigma\right)^{-1}\left(HUH^\top \!  +\! \Sigma-HUH^\top \right)\\
&= HUH^\top \left(HUH^\top   + \Sigma\right)^{-1}\Sigma\\
&= \Sigma\left(HUH^\top   + \Sigma\right)^{-1}HUH^\top \\
&= \Big( \left( HUH^\top   \right)^{-1} \left(HUH^\top   + \Sigma\right) \Sigma^{-1} \Big)^{-1}\\
\Xi_{i,k} &= \left[\Sigma^{-1}+ \left(HUH^\top \right)^{-1}\right]^{-1}.
\end{align*}
These manipulations are legal because covariance matrices are positive definite and therefore symmetric and invertible.
\end{proof}

\subsection{The Next Best View Problem}\label{sec:NBVprob}
Suppose there are $k-1$ available observations of the group of targets in $\ccalN$, and let
\begin{equation}\label{eq:s_objective}
HU_{s,k|k-1}H^\top \; \textrm{with} \;\; s= \argmax_{j \in \ccalN} \left\{ {\bf tr}\, \left[HU_{j,k|k-1} H^\top \right] \right\}
\end{equation}
denote the predicted covariance of the worst localized target and
\begin{equation}\label{eq:c_objective}
HU_{c,k|k-1}H^\top  = \frac{1}{n} \sum_{i \in \ccalN} H U_{i,k|k-1}H^\top .
\end{equation}
denote the average of all predicted target covariances at iteration $k$. The problem that we address in this paper is as follows.

\begin{problem}[Next-Best-View]\label{problem}
Given the predicted covariance of the worst localized target $HU_{s,k|k-1}H^\top $ (respectively, the average of the targets'  predicted covariances $HU_{c,k|k-1}H^\top $) and the predicted next location $\hbbz_{s,k | k-1}$ of target $s$ (respectively, the average of the targets' predicted locations $\hbbz_{c,k | k-1}$), determine the next pose of the stereo rig $(\bbr(t_k), R(t_k))$ so that ${\bf tr}[\Xi_{s,k}]$ (respectively, ${\bf tr}[\Xi_{c,k}]$) is minimized.
\end{problem}

To solve Problem~\ref{problem} we make the following assumptions:
\begin{enumerate}
\item[(A1)] noise is dominated by quantization of pixel coordinates;
\item[(A2)] correct correspondence of the targets between the images in the stereo rig exists;
\item[(A3)] if targets are in the field of view of the cameras, then they are not occluded by any obstacle in space.
\end{enumerate}

Assumption~(A1) allows us to isolate, analyze, and control the effect of pixelation noise on the target localization process. While other sources of noise do exist, pixelation noise does in fact dominate for small disparities, e.g., when the camera is far away from the targets. The effect of other noise sources can be critical for the stability of the KF, and we discuss a novel data-driven approach to obtain empirical models of these uncertainties in Section~\ref{sec:exp}.
Assumptions~(A2)~and~(A3) allow us to simplify the problem formulation. It is well known that correspondence and occlusion are both important problems and, being such, have received significant attention in the computer vision literature, see e.g., \cite{scharstein2002} and the references therein. In this paper, assumptions~(A2)~and~(A3) allow us to obtain analytic and computationally efficient solutions to the Next-Best-View and target localization problem using exact models of stereo vision sensors. In situations where correspondence errors and occlusions do not raise significant challenges, e.g., for sparse target configurations, our approach can have significant practical applicability.

In problem~\ref{problem}, we have chosen the trace as a measure of uncertainty among other choices, such as the determinant or the maximum eigenvalue. (A similar choice was made by \cite{ponda09}.) 
\cite{wenhardt07} shows that all such criteria behave similarly in practice. 
Since minimization of ${\bf tr}[\Xi_{s,k}]$ is associated with improving localization of the worst localized target, we call it the \emph{supremum objective}. 
We call minimization of ${\bf tr}[\Xi_{c,k}]$ the \emph{centroid objective}. 
$\Xi_{s,k}$ will depend only on the predicted next position of the worst localized target, which we denote by $\bbp_{s,k}$, but $\Xi_{c,k}$ will depend on the {predicted next} positions $\bbp_{i,k}$ of all $i \in \ccalN$.  

Attempting to solve Problem~\ref{problem} by simultaneously controlling the covariances of all targets requires a nonconvex constraint to maintain consistency between images.
We note that, when we employ the supremum or centroid objective, the decision process comprises two nonlinear procedures: triangulation and Kalman Filtering. 

\section{Controlling the Relative Frame} \label{sec:potdes}
Assume that $k-1$ observations are already available. Our goal in this section is to determine the next best target locations $\bbp_{s,k}$ or $\bbp_{c,k}$ on the relative camera frame so that if a new observation is made with the targets at these new relative locations, the fused localization uncertainty, which is captured by $\Xi_{s,k}$ or $\Xi_{c,k}$, is optimized. For this, we need to express the instantaneous covariance $\Sigma_{i,k}$ of target $i$ as a function of the relative position $\bbp_{i,k}$.
To simplify notation, in this section we drop the subscripts $s, c,$ and $o$.
We will also drop the subscript $k$ when no confusion can occur. 

From \eqref{eq:bbp}, we know that $\bbp$ depends on the noisy vector $(\chkx_L,\chkx_R,\chky)$, which we assume has some known or easily estimated covariance $Q$.
\footnote{Recall that we approximate the uniform pixelation noise as Gaussian, hence the approximate nature of $Q$.}  
In the experiments of Section~\ref{sec:exp}, we propose a new data-driven linear model to estimate $Q$.
Let $J$ be the Jacobian of $\bbp\triangleq \bbp(x_L,x_R,y)$ evaluated at the point $(\chkx_L,\chkx_R,\chky)$, given by
\begin{equation} \label{eq:jac}
J =\frac{b}{(\chkx_L-\chkx_R)^{2}} \begin{bmatrix}  -\chkx_R & \chkx_L & 0 \\ -\chky & \chky & \chkx_L-\chkx_R \\ -f & f & 0 \end{bmatrix}.
\end{equation}
Then, the first order (linear) approximation of $\bbp $ about the point $(\chkx_L,\chkx_R,\chky)$ is
\begin{equation} \label{eq:linearization}
\bbp(x_L,x_R,y) \approx \bbp(\chkx_L,\chkx_R,\chky) + J  \begin{bmatrix} \chkx_L-x_L \\ \chkx_R-x_R \\ \chky-y \end{bmatrix} \!\!.
\end{equation}
Since $\bbp (\chkx_L,\chkx_R,\chky)$ corresponds to the current mean estimate of target coordinates, it is constant in \eqref{eq:linearization}.  Therefore, the covariance of $\bbp $ in the relative camera frame is $J QJ ^\top $.
Fusing covariance matrices as in Lemma~\ref{lem:kfchange} requires that they are  represented in the same coordinate system. To represent the covariance $J QJ ^\top $ in global coordinates, we need to rotate it by an amount corresponding to the camera's orientation at the time this covariance is evaluated. Assuming that consecutive observations are close in space, so that the camera makes a small motion during the time interval $[t_{k-1},t_{k}]$, we may approximate the camera's rotation $R(t)$ at time $t\in[t_{k-1},t_{k}]$ by its initial rotation $R(t_{k-1})$. 
We note that this approximation will be inaccurate if the robot moves long distances between consecutive observations.
For the use case discussed in Section~\ref{sec:exp}, the robot takes multiple observations per second, so this approximation is not an issue.
Denoting $R(t_{k-1})$ by $R$, the instantaneous covariance of $\bbp $ can be approximated by
\begin{equation} \label{eq:instcov}
\Sigma  = \cov[\bbp(\chkx_L,\chkx_R,\chky)]  \approx  RJ QJ ^\top R^\top  \!\!.
\end{equation}
In view of \eqref{eq:bbp} and \eqref{eq:jac}, the covariance in \eqref{eq:instcov} is clearly a function of the target coordinates on the relative image frame.
Using this model of measurement covariance, we define the uncertainty potential
\begin{equation} \label{eq:h}
h(\bbp) = {\bf tr}\big\{ \Xi \big\},
\end{equation}
Then, the next best view vector that minimizes $h$ can be obtained using the gradient descent update
\begin{equation}\label{eqn_p_update}
 \bbp_{k} =\bbp_{k-1} -   K \int_{0}^{T} \nabla_{\bbp} h\left(\bbp (\tau)\right) d\tau,
\end{equation} where $K$ is a gain matrix.
The length $T>0$ of the integration interval is chosen so 
that the distance between $\bbp_{k}$ and $\bbp_{k+1}$ is less than the maximum distance the robot can travel before another NBV is calculated at time $t_{k}$.
The following result provides an analytical expression for the gradient of the potential $h$ in \eqref{eqn_p_update}.
\begin{prop} \label{prop:grad}
The $j$-th coordinate of the gradient of $h$ with respect to $\bbp $ is given by
\begin{equation} \label{eq:gradient f_h}
\frac{\partial h}{\partial [\bbp ]_j}  ={\bf tr} \left\{\Sigma ^{-1} \Xi^{2}\Sigma ^{-1}  \frac{\partial \Sigma }{\partial  \left[\bbp \right]_j}  \right\},
\end{equation}
where $\frac{\partial \Sigma }{\partial \left[ \bbp \right]_j}$ is the partial derivative of $\Sigma $ with respect to the $j$-th coordinate of $\bbp $, and $j=x,y,z$, corresponds to the three dimensions of $\bbp $.
\end{prop}
It is not difficult to show that
\begin{equation}\label{eq:grad_mid}
\frac{\partial h}{\partial [\bbp ]_j} = -{\bf tr} \left\{ \Xi ^2 \frac{\partial \left[ ( HU H^\top  )^{-1} + \Sigma ^{-1} \right]}{\partial \left[\bbp \right]_j}  \right\}.
\end{equation}
Note that the covariance of all prior fused measurements $HU H^\top $ is a constant with respect to the next best view $\bbp $ and, therefore, its derivative with respect to $\bbp $ is zero, i.e., $\partial \left(HU H^\top \right)^{-1} / \partial\left[ \bbp  \right]_j = 0$.
The derivative of $\Sigma ^{-1}$ with respect to $\left[ \bbp  \right]_j$ leads to an expression for the derivative in the right-hand-side of \eqref{eq:grad_mid} that retreives \eqref{eq:gradient f_h}.

In what follows, we apply the chain rule to calculate $\partial \Sigma  / \partial \left[\bbp \right]_j$ in \eqref{eq:gradient f_h}.  In particular,
since we hold $R$ constant during the relative update, we have that the partial derivatives of $\Sigma $ in the directions $[\bbp ]_j$ for $j=x,y,z$ are taken only with respect to the entries of $J Q J^\top $, i.e.,
\begin{equation}\label{eq:jac_deriv}
\frac{\partial \Sigma }{\partial \left[\bbp \right]_j} 
=
R \left( \frac{\partial J}  {\partial\left[ \bbp \right]_j}  QJ^\top + J Q \frac{\partial J^\top}  {\partial\left[ \bbp \right]_j}   \right) R^\top .
\end{equation}
Then, using the chain rule,
\begin{equation}\label{eq:partials}
\frac{\partial J  }{\partial\left[ \bbp \right]_j} =
		\frac{\partial J }{\partial x_L }\frac{ \partial x_L}{ \left[\bbp \right]_j}+
		\frac{\partial J }{\partial x_R }\frac{ \partial x_R}{ \left[\bbp \right]_j}+
		\frac{\partial J}{\partial y  }\frac{ \partial y }{ \left[\bbp \right]_j}.
\end{equation}
The need arises to express the pixel coordinate tuple $(x_L,x_R,y )$ as a function of the location of target in relative coordinates $\bbp $.  
This is available via the inverse of \eqref{eq:bbp}, given by
\begin{equation}\label{eq:invbbp}
\mat{x_L \\ x_R \\ y}
=
\frac{f}{[\bbp]_z} \mat{[\bbp ]_x + \frac{b}{2} \\ [\bbp ]_x - \frac{b}{2} \\ [\bbp ]_y }.
\end{equation}
Then, \eqref{eq:partials} can be evaluated by finding the partial derivative of $J$ with respect to $(x_L,x_R,y )$ and
the partial derivatives of the entries of \eqref{eq:invbbp} with respect to each coordinate of $\bbp $.  Using these derivatives, all terms in \eqref{eq:jac_deriv} are accounted for, which completes the proof of Proposition~\ref{prop:grad}.

\section{Controlling the Global Frame} \label{sec:realize}
The update in \eqref{eqn_p_update} provides the desired change in relative target coordinates $\bbp_{o,k} -\bbp_{o,k-1} $ of target $o$ in the camera frame, where $o$ stands for `objective' and can be either $s$ or $c$, depending on the objective defined in Problem~\ref{problem}.
Our goal in this section is to determine a new camera position $\bbr(t_k)$ and orientation $R(t_k)$ in space that realizes the change in view, effectively arriving at the Next Best View of the target located at $\hbbx_o.$
Transforming the change of view into global coordinates, the goal position $\bbr^*$ is defined as
\begin{equation}
\bbr^* \triangleq \bbr(t_{k-1}) + R (t_{k-1})( \bbp_{o,k} -\bbp_{o,k-1}).
\end{equation}
The ability to rotate the camera in addition to translating it means that there are infinitely many poses in the global frame that realize the NBV in relative coordinates.
The goal orientation is defined to be any pose such that the point $\hbbx_o$ lies on the $z$-axis of the camera relative coordinate system, i.e., the camera is looking straight at the centroid (or supremum) target location.
To achieve this new desired camera position and orientation, we define the following potential which we seek to minimize:
\begin{align}
\label{eq:psi_potential}
\psi  \left(\bbr,R \right) &= 
\overbrace{
\norm{\bbr - \bbr^* }^2}^\text{position}  + \overbrace{\norm{R^\top  \hbbz - \bbe_3}^2}^\text{orientation},
\end{align}
where
\begin{align}\label{eq:unit-vecs}
\hbbz = \frac{ \hbbx_o - \bbr^* }{\norm{\hbbx_o - \bbr^*}} \text{ and }
\bbe_3= \left[0 \; 0 \;  1\right]^\top.
\end{align}
In \eqref{eq:unit-vecs}, $\hbbz$ is the direction in global coordinates from the desired robot position $\bbr^*$ to the estimated target-objective location $\hbbx_o,$
and $\bbe_3$ is the unit vector in the direction of the robot's view in relative coordinates, defined to be the $z$-axis.
Note that the robot and target cannot be located at the same point, because this would violate field of view constraints.

To minimize $\psi$, we define the following gradient flow for all time $t \in [t_{k-1},t_k]$
\begin{subequations}\label{eq:rRdiff}
\begin{align}
\dot{\bbr} &= -\nabla_\bbr {\psi}(\bbr,R)  \label{eq:rdiff} \\
\dot{R} &= -R \nabla_R {\psi}(\bbr,R), \label{eq:Rdiff}
\end{align}
\end{subequations}
in the joint space of camera positions $\reals^3$ and orientations $SO(3)$, where $\nabla_\bbr {\psi}$ and $\nabla_R{\psi}$ are the gradients of $\psi$ with respect to $\bbr$ and $R$. After initializing the gradient flow \eqref{eq:rRdiff} at $\big(\bbr(t_{k-1}),R(t_{k-1})\big)$, the following lemma shows that if $R(t_{k-1})\in SO(3)$ and $R(t)$ evolves as in \eqref{eq:Rdiff} and $\nabla_R {\psi}(\bbr,R)$ is skew-symmetric, then $R(t)\in SO(3)$ for all time $t\in[t_{k-1},t_{k}]$; see \cite{Zavlanos2008}.

\begin{lem} \label{lem:flow}
 Let $\Omega(t) $ be skew-symmetric $\forall \, t \geq 0$ and define the matrix differential equation $\dot{R}(t) = R(t)\Omega(t)$.  Then, $R(t) \in SO(n) \, \forall \, t \geq 0$ if $R(0) \in SO(n)$.
 \end{lem}
 In other words, the gradient flow \eqref{eq:Rdiff} is implicitly constrained to the set of Special Euclidean transformations during the minimization of $\psi$.

\subsection{Closed Form Motion Controllers}
In the remainder of this section we provide analytic expressions for the gradients in \eqref{eq:rRdiff}. 
We also use these expressions to show that the closed loop system \eqref{eq:rRdiff} minimizes ${\psi}$.
The first proposition identifies the gradient of $\psi$ with respect to $R.$
To prove it, we use the matrix inner product $\langle A,B\rangle={\bf tr}(A^\top B)$, which has the following property.
\begin{lem}\label{lem:skew}
For any square matrix $A$ and skew-symmetric matrix $\Omega$ of appropriate size, 
$2 \left\langle A, \Omega \right\rangle 
=
 \left\langle A - A^\top, \Omega \right\rangle.
$
\end{lem}
\begin{proof}
We have that
$
2 \left\langle A, \Omega \right\rangle = \left\langle A, \Omega \right\rangle + \left\langle  \Omega, A \right\rangle = \bbt \bbr (A^\top \Omega + \Omega^\top A)= \bbt \bbr ( (A^\top - A )\Omega)= \left\langle A - A^\top, \Omega \right\rangle .
$
\end{proof}

\begin{prop} \label{lem:dpdR}
The gradient of $\psi$ with respect to $R$ is given by the skew-symmetric matrix
\begin{equation}\label{eq:nabla_psi_R}
\nabla_R\psi=
 R^\top  \hbbz
(R^\top  \hbbz - \bbe_3)^\top
-
(R^\top  \hbbz - \bbe_3 )
   \hbbz^\top R .
\end{equation}
\end{prop}

\begin{proof} 
For any skew symmetric matrix $\Omega$,
\agn*{
\norm{(R(I+\Omega))^\top  \hbbz - \bbe_3}^2
&=
\norm{
R^\top  \hbbz - \bbe_3 -\Omega  R^\top  \hbbz}^2.
}
Let $\bbv = R^\top  \hbbz - \bbe_3$ to simplify notation.  
Using the first order approximation of the neighborhood of the rotation matrix $R(\Omega) \approx I + \Omega$, where $\Omega$ is skew-symmetric, and using Lemma~\ref{lem:skew} along with the basic properties of inner products, we have that
\agn*{
\psi(\bbr, R(I + &\Omega)) \!
= \!\norm{\bbr -\bbr^*}^2 \!\!+ \norm{\bbv-\Omega  R^\top  \hbbz}^2
\\
&= \!\norm{\bbr -\bbr^*}^2 \!\!+ \norm{\bbv}^2\!\! -2 \left\langle \bbv,  \Omega  R^\top  \hbbz \right\rangle  \! + \!  o \left( \norm{\Omega} \right)
\\
&=\psi(\bbr, R) -2 \left\langle
\bbv  \hbbz^\top R,
\Omega  
\right\rangle + o \left( \norm{\Omega} \right)
\\
&= 
\psi(\bbr, R) \!  + \!  \left\langle
  R^\top \hbbz \bbv^\top\!  \!
 \! - \! 
 \bbv  \hbbz^\top R 
  ,
\Omega  
\right\rangle  \! +  \! o \left( \norm{\Omega} \right)
,
}
from which we identify $  R^\top \hbbz \bbv^\top 
-
 \bbv  \hbbz^\top R $ as $\nabla_R\psi (\bbr, R),$ and the result follows immediately.
\end{proof}
Additionally, we have from elementary calculus that
\begin{align}\label{eq:nabla_psi_r}
\nabla_\bbr \psi(\bbr,R) 
&= 2 (\bbr - \bbr^*)
\\
&=  2 \left( \bbr  -   \bbr(t_{k-1})   -   R (t_{k-1})( \bbp_{o,k}   -\bbp_{o,k-1}) \right).\nonumber
\end{align}

The following result shows that the closed loop system \eqref{eq:rRdiff} is globally asymptotically stable about the minimizers of $\psi$.
\begin{thm}\label{thm:convergence}
The trajectories of the closed loop system \eqref{eq:rRdiff} globally converge to the set of minimizers of the function $\psi$.
\end{thm}
\begin{proof}
By inspection of \eqref{eq:psi_potential}, $\psi(\bbr, R) \ge 0,$ with equality if and only if $R^\top \hbbz = \bbe_3$ and $ \bbr=\bbr^*$.
In the remainder of the proof, we show that $\psi$ is a suitable Lyapunov function for the closed loop system \eqref{eq:rRdiff}, and the set of equilibrium points is exactly the set of minimizers of $\psi.$
To begin, let $\bbv$ be defined as above, so that
\begin{align}
&\dot{\psi} (\bbr, R) \!
=  \! 2 \left\langle \bbr \!   - \!   \bbr^*, \dot{\bbr} \right\rangle
+2 \left\langle R^\top \hbbz - \bbe_3, \dot{R}^\top \hbbz \right\rangle\nonumber
\\ 
 &=\!  2\left\langle \bbr  \! - \!  \bbr^*, \!-\!\nabla_\bbr \psi(\bbr, R) \right\rangle
+2 \left\langle \bbv,  (-R \nabla_R \psi(\bbr, R) )^\top \hbbz \right\rangle\nonumber
\\ 
&= \!  2\left\langle \bbr  \! - \!  \bbr^*, - 2 (\bbr \!  -  \! \bbr^*) \right\rangle
+2  \left\langle \bbv,  (R^\top   \!  \hbbz \bbv^\top \!  \!  \!  - \!  \bbv \hbbz^\top  \!  R)R^\top  \! \hbbz \right\rangle\nonumber
\\ 
&= \!  - 4\norm{\bbr   \!  -   \!  \bbr^*}^2  \!  
 +  \!   2 \left( \! \left\langle \bbv,  R^\top   \!  \hbbz \bbv^\top \! R^\top  \!   \hbbz \right\rangle   \!  -   \!  \left\langle \bbv ,\bbv \hbbz^\top  \!  R R^\top  \! \hbbz \right\rangle \!\right) \nonumber
\\
&= \!  - 4\norm{\bbr \!  - \!   \bbr^*}^2
+2\left( \left(\bbv^\top R^\top \hbbz \right)^2  - \norm{\bbv}^2 \right).
\label{eq:dpsi-negative}
\end{align}
The Cauchy-Schwartz inequality implies that
\[
\left(\bbv^\top R^\top \hbbz \right)^2 
 \le  \norm{ \bbv}^2 \norm{ R^\top \hbbz}^2 = \norm{ \bbv}^2 ,
 \]
 so that \eqref{eq:dpsi-negative} is the sum of two nonpositive terms.
 Thus, $\dot{\psi} \le 0$, with equality if and only if both of the nonpositive terms are zero.
 In particular, $\dot{\psi} (\bbr, R)=  0$ if and only if $ \bbr = \bbr^*$ and
\agn*{
\norm{\bbv}^2 &= (\hbbz^\top R \bbv )^2 
\\
(\hbbz^\top  R - \bbe_3^\top)(R^\top \hbbz - \bbe_3) &= (\hbbz^\top R (R^\top \hbbz - \bbe_3) )^2  
\\
2(1-\hbbz^\top R \bbe_3) &= (1 -\hbbz^\top R \bbe_3)^2  
\\
1 -  (\hbbz^\top R \bbe_3)^2&= 0,
}
which implies $\hbbz^\top R= \bbe_3^\top$ for all critical points. 
Invoking the Lyapunov Stability Theorem, the result follows.\end{proof}
Note that the system \eqref{eq:rRdiff} evolves during the time interval $[t_{k-1},t_k]$, until a new observation of the targets is made at time $t_k$. This time interval might not be sufficient for the camera to realize exactly the NBV. Nevertheless, Theorem~\ref{thm:convergence} implies that at time $t_k$, the position and orientation of the camera is closer to desired NBV than it was at time $t_{k-1}$. By appropriately choosing the length of the time interval $[t_{k-1},t_k]$, we may ensure that for practical purposes the camera almost realizes the NBV.

\section{Performance of the Integrated Hybrid System}\label{sec_simulations}
In this section, we illustrate our approach in computer simulations.
We begin by discussing a practical method of how to incorporate field of view constraints in the hybrid system, which is used in our simulations and experimental results.

\subsection{Incorporating Field of View Constraints}
\label{sec:fov}
For a 3D point to appear in a given image, that point must lie within the field of view of both cameras in the stereo pair as the robot rotates and translates in an effort to minimize \eqref{eq:psi_potential}.  We assume that the (L) and (R) cameras have identical square sensors with a 70$^\circ$ field of view, which, when combined with the image width $w$, determines the focal length $f$.
Let
\[
\ccalS= \set{ [x,y,z]^\top \in \reals^3 \colon \abs{x} \leq \frac{wz - bf}{2f}   ,\:  \abs{y} \leq  \frac{zw}{2f},\: z > \frac{bf}{w} }
\]
denote the set of points in relative coordinates that are visible to both cameras in the pair. This set is the intersection of two pyramids facing the positive $z$ direction with vertices located at the two camera centers.  Figure~\ref{fig:fov_3d} visualizes the set $\ccalS$ in two dimensions (blue shaded region).  Note that the intersection of the two pyramids is located at  $z = \frac{bf}{w}$, and therefore any point with $z < \frac{bf}{w}$ can not be in view of both cameras.
\begin{figure}[t]
 \centering
 \includegraphics[width=7cm]{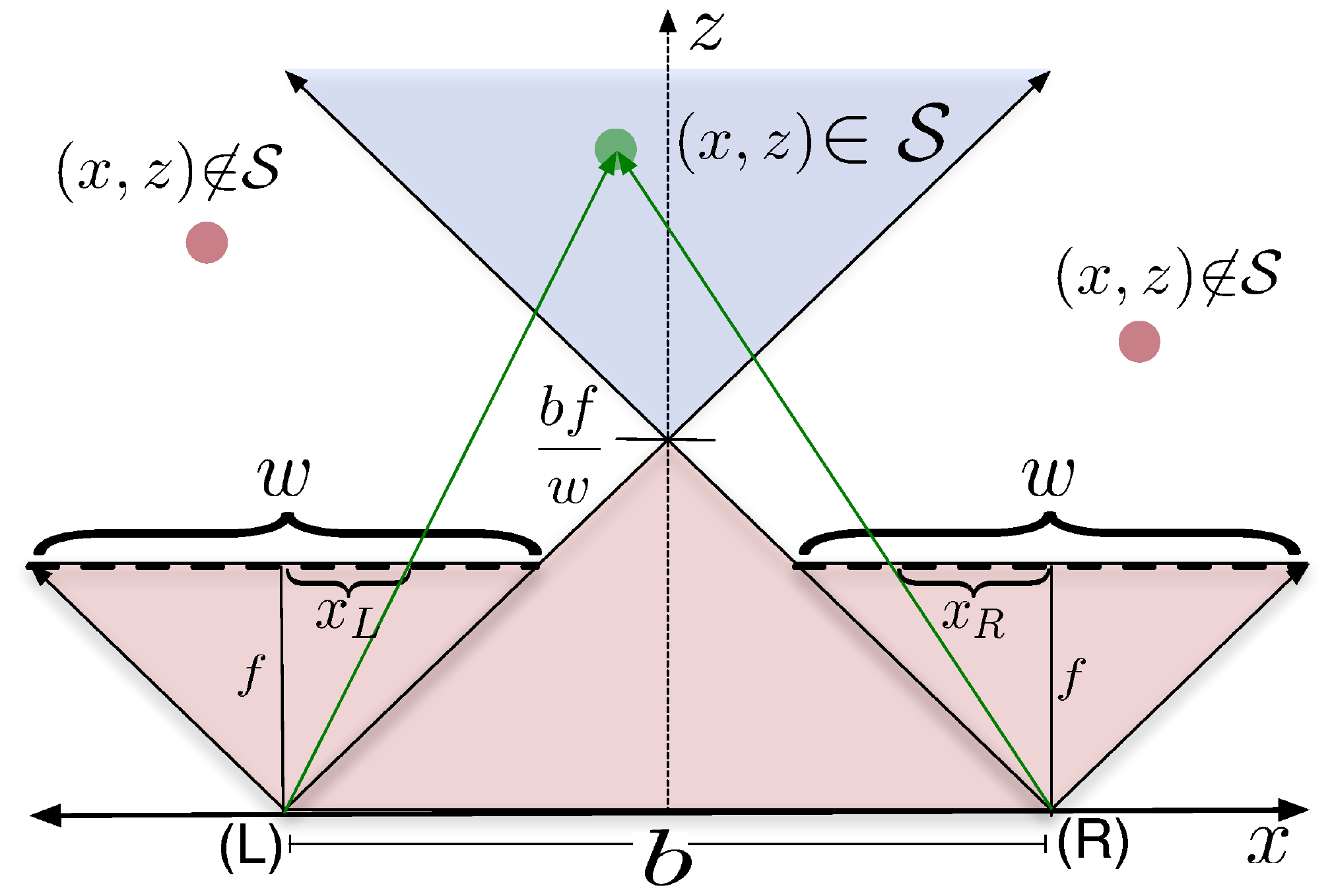}
 \caption{The field of view for a stereo camera in the $xz$ plane.  The field of view in the $yz$ plane is similar.}  \label{fig:fov_3d}
\end{figure}

Maintaining all targets within the FoV $\ccalS$ requires that the camera positions and orientations evolve in the set
\begin{equation} \label{eq:ccalD}
\ccalD = \set{ (\bbr,R) \in \reals^3 \times SO(3) \, : \, \set{\bbp_i}_{i \in \ccalN}  \in \ccalS }
\end{equation}
for all time. To ensure invariance of the set $\ccalD$, we define the potential functions
\begin{subequations} \label{phi}
\begin{align}
\phi_{i1}(\bbr,R) &= \left(\frac{w[\bbp_i]_z - bf}{2f}\right)^2 - [\bbp_i]_x^2,\\
 \phi_{i2}(\bbr,R) &= \left(\frac{w[\bbp_i]_z}{2f}\right)^2 - [\bbp_i]_y^2,\\
\phi_{i3}(\bbr,R) &= [\bbp_i]_z^2 - \left(\frac{bf}{w}\right)^2,
\end{align}
\end{subequations}that are positive if $(\bbr,R) \in \ccalD$, where
$[\bbp_i]_x$, $[\bbp_i]_y$, and $[\bbp_i]_z$ are the $x$, $y$, and $z$ coordinates of target $i$ in the relative camera frame, that can be expressed in terms  of the camera position and orientation as
\begin{subequations} \label{eq:coord_transformations}
\begin{align}
[\bbp_i]_x& =  \left\langle \bbe_1,R^\top (\hat{\bbx}_i - \bbr)\right\rangle  ,\\
[\bbp_i]_y &=  \left\langle \bbe_2,R^\top (\hat{\bbx}_i - \bbr)\right\rangle  ,\\
[\bbp_i]_z &= \left\langle \bbe_3,R^\top (\hat{\bbx}_i - \bbr)\right\rangle  ,
\end{align}
\end{subequations}
where $\bbe_1$, $\bbe_2$, and $\bbe_3$ are the unit vectors in the standard basis.  Then, given an estimate of target locations $\hat{\bbx}_i$ for $i=1,\dots,n$, we augment the potential $\psi$ from \eqref{eq:psi_potential} by adding barrier functions $1/\phi_{ij}$ that will grow without bound anytime a target is close to the boundary of the feasible set $\ccalD$.  
The repulsive force supplied by $\phi_i$ is regulated by a user defined penalty parameter $\rho>0$.
The artificial potential function, incorporating the desired FoV constraints, is given by
\begin{equation} \label{eq:objectivehat}
\hat{\psi}\left(\bbr,R \right) = \psi\left(\bbr,R\right) + \frac{\rho}{n}  \sum_{i \in \ccalN}\sum_{j = 1}^3 g(\phi_{ij}),
\end{equation}
where 
$g \colon \reals \to \reals$ is a barrier potential, and
multiplication by $1/n$ ensures that the number of targets does not affect the strength of the penalty.  The penalty parameter $\rho$ is set sufficiently small so that $\hat{\psi}$ approximates $\psi$ when $(\bbr,R)$ is in the interior of $\ccalD$ while maintaining that $\hat{\psi}$ becomes extremely large for $(\bbr,R)$ that approach the boundary of $\ccalD$.  Replacing $\psi$ with $\hat{\psi}$ in the gradient flow in Algorithm~\ref{alg1} provides the desired potential that realizes the NBV and respects FoV constraints.
In the simulations, we set $g(a) = \frac{1}{a}$.

In what follows we derive analytical expressions for the gradients of $\hat{\psi}$.  In particular, we have that
\begin{subequations}
\label{eq:grad_hpsi}
\begin{align}
\nabla_\bbr \hat{\psi}\left(\bbr,R \right) = \nabla_\bbr  \psi + \frac{\rho}{n} \sum_{i \in \ccalN}\sum_{j = 1}^3 g' (\phi_{ij})\nabla_\bbr\phi_{ij}, \\
\nabla_R \hat{\psi}\left(\bbr,R \right) = \nabla_R  \psi + \frac{\rho}{n} \sum_{i \in \ccalN}\sum_{j = 1}^3g' (\phi_{ij})   \nabla_R \phi_{ij}.\label{eq:grad_hpsi_dR}
\end{align}
\end{subequations}
The derivative of the barrier function, $g'$, is available from elementary calculus.
The gradients in \eqref{eq:grad_hpsi} with respect to $\bbr$ and $R$ can be obtained by application of the chain rule as
\begin{align}
\nabla_\bbr\phi_{ij}  \! &=  \! \frac{\partial \phi_{ij}}{\partial [\bbp_i]_x}\nabla_\bbr [\bbp_i]_x +\frac{\partial \phi_{ij}}{\partial [\bbp_i]_y}\nabla_\bbr [\bbp_i]_y +\frac{\partial \phi_{ij}}{\partial [\bbp_i]_z}\nabla_\bbr [\bbp_i]_z\nonumber\\
\nabla_R\phi_{ij}  \! &=  \!  \frac{\partial \phi_{ij}}{\partial [\bbp_i]_x}\nabla_R [\bbp_i]_x  \! +  \!\frac{\partial \phi_{ij}}{\partial [\bbp_i]_y}\nabla_R [\bbp_i]_y   \! +  \! \frac{\partial \phi_{ij}}{\partial [\bbp_i]_z}\nabla_R [\bbp_i]_z.\label{chain}
\end{align}
The coefficients in \eqref{chain} can be obtained by differentiating \eqref{phi}.  The following propositions provide the gradients of $[\bbp_i]_x, [\bbp_i]_y$, and $[\bbp_i]_z$ with respect to $R$ and $\bbr$.
\begin{prop} \label{lem:Dxyz_DR}
The gradients of $[\bbp_i]_x, [\bbp_i]_y$, and $[\bbp_i]_z$ with respect to $R$ are given by the skew symmetric matrices
\begin{subequations}
\begin{align}
\nabla_R [\bbp_i]_x &= \frac{1}{2} \left[ R^\top (\hat{\bbx}_i - \bbr) \bbe_1^\top  - \bbe_1 (\hat{\bbx}_i - \bbr)^\top R  \right], \label{Dx_DR}\\
\nabla_R [\bbp_i]_y &=\frac{1}{2}  \left[  R^\top (\hat{\bbx}_i - \bbr) \bbe_2^\top  -\bbe_2 (\hat{\bbx}_i - \bbr)^\top R\right] ,\\
\nabla_R [\bbp_i]_z &= \frac{1}{2}  \left[ R^\top (\hat{\bbx}_i - \bbr) \bbe_3^\top  -\bbe_3 (\hat{\bbx}_i - \bbr)^\top R \right].
\end{align}
\end{subequations}\end{prop}
\begin{proof}
The procedure here is nearly identical to the method in Proposition~\ref{lem:dpdR}.
Specifically,
\agn*{
[\bbp_i]_x (\bbr,R(I+\Omega))
&=\!
\left\langle
\bbe_1, (R(I+\Omega))^\top (\hbbx_i -\bbr)
\right\rangle
\\
&=\!
\left\langle
\bbe_1,  (I + \Omega)^\top R^\top  (\hbbx_i -\bbr)
\right\rangle
\\
&=\!
\left\langle
\bbe_1, R^\top  (\hbbx_i -\bbr)
\right\rangle
\! -\!
\left\langle
\bbe_1,\Omega R^\top  (\hbbx_i  \! -  \! \bbr)
\right\rangle
\\
&=\!
[\bbp_i]_x (\bbr,R)
\! -\!
\left\langle
\bbe_1 (\hbbx_i -\bbr)^\top R,\Omega 
\right\rangle.
}
Again we can use Lemma~\ref{lem:skew} to obtain that
\agn*{
\left\langle
\bbe_1 (\hbbx_i \! - \! \bbr)^\top R,\Omega 
\right\rangle \!
&= \!
\frac{1}{2}
\left\langle
\bbe_1 (\hbbx_i \! -\! \bbr)^\top R - R^\top  (\hbbx_i -\bbr)\bbe_1^\top, \Omega 
\right\rangle,
}
from which we can identify the gradient as the term linear in $\Omega,$ and the proof follows.
The gradients of the other two coordinates are found analogously.
\end{proof}

Note that the gradients of the functions $[\bbp_i]_x, [\bbp_i]_y$, and $[\bbp_i]_z$ with respect to $R$ are skew-symmetric, as required for \eqref{eq:Rdiff} to ensure that $R\in SO(3)$ for all time; see Lemma~\ref{lem:flow}.
From elementary calculus, the gradients of $[\bbp_i]_x, [\bbp_i]_y$, and $[\bbp_i]_z$ with respect to $\bbr$ are
\begin{equation}\label{d_wr_r}
\nabla_\bbr [\bbp_i]_x \! =  \!  -R \bbe_1 ,\, \nabla_\bbr [\bbp_i]_y  \! = \!  -R \bbe_2 ,  \,
\nabla_\bbr [\bbp_i]_z  \! =  \! -R \bbe_3 .
\end{equation}

\subsection{Outline of Controller}

\begin{algorithm}[t]
\caption{Hybrid control in the relative and global frames. }
\label{alg1}
\begin{algorithmic}[1]
\REQUIRE A position $\bbr(t_{k-1})$ and orientation $R(t_{k-1})$ of the camera and estimated positions $\hat{\bbx}_{i,k-1}$ of the targets.
\STATE \label{step:rel} Find the next best view associated with objective ``$o$'' according to equation \eqref{eqn_p_update}:
\begin{equation*}
 \bbp_{o,k} =\bbp_{o,k-1} - K \int_{0}^{T} \nabla h\left(\bbp_{o}(\tau)\right) d\tau.
\end{equation*}
\STATE Move the camera according to the system \eqref{eq:rRdiff}:
\begin{align*}
\dot{\bbr} &= - \nabla_\bbr {\hat{\psi}}(\bbr,R),  \\
\dot{R} &= - R \nabla_R \hat{\psi}(\bbr,R),
\end{align*}
for a time interval of length $t_{k}-t_{k-1}$ in order to realize the next best view $\bbp_{o,k}$ obtained from step 1.
\STATE At time $t_{k}$ observe targets and incorporate new estimates and covariances into KF as in \eqref{eq:zpred} and \eqref{eq:covfinal}.  Increase the observation index $k$ by 1 and return to step 1.
\end{algorithmic}
\end{algorithm}

Algorithm~\ref{alg1} outlines the hybrid controller developed in Sections~\ref{sec:potdes}~and~\ref{sec:realize}.  After initialization, Step 1 determines the NBV according to either the \emph{supremum objective} or the \emph{centroid objective}.  Given a frame rate and sensor speed, we set the integration interval $T$ so that the distance between $\bbp_{o,k-1}$ and $\bbp_{o,k}$ is the maximum distance the camera can travel before making a new observation.  Each time a new observation is made, Step 1 returns a new NBV $\bbp_{o,k}$, which constitutes a discrete switch in the potential $\hat{\psi}$ in Step 2.  This switch results in a new motion plan that guides the robot to a position and orientation that realizes the new NBV.  The camera moves according to Step 2 until a new measurement is taken, at which point we set $k:=k+1$ and return to Step 1.

\subsection{Static Target Localization}
\label{sec:subsec_stat_target_local}
We begin this section by illustrating our approach for a simple scenario involving an array of five stationary targets in two dimensions.  
In this case, the mobile stereo camera effectively has only two motion primitives available: ``reduce depth'' and ``diversify the viewing angle.''  
Thus, the optimal controller will be a state-dependent combination of these two primitives, which should emerge naturally by minimizing the objective function we have described herein.  
For comparison, we present a ``straight baseline'' and a ``circle baseline,'' which exclusively utilize one of these two motion primitives.
Specifically, the circle baseline moves the robot in a circle about the cluster of targets, and the straight baseline drives the robot closer to the targets. 
We require that all methods travel the same distance in each iteration, except for the straight baseline that stops moving once FoV constraints tend to become violated.
To test the validity of our assumption that pixel-noise due to quantization is uniform on the image plane, all simulated observations in this section have been quantized at the pixel level.

Figure~\ref{fig:traj_static} shows robot trajectories for this simple example.
It can be seen that reducing range results in slightly better short-term performance for this situation, but once FoV constraints are nearly violated, repetitive observations from the same spot have correlated noise, leading to divergence of the KF; a similar behavior can be observed in Figure~\ref{fig:errors_and_trace}, which refers to the 3D example bellow.
On the other hand, the circle baseline and the supremum and centroid objectives continuously move the camera, so the i.i.d. noise assumption is not violated and the localization error throughout the simulation keeps being reduced, as can be seen in the shrinking confidence ellipses in Figure~\ref{fig:ell}.  
By combining the two motion primitives automatically, the supremum and centroid objectives reduce noise by an order of magnitude compared to the straight baseline (before its KF diverges) after around 23 observations.

\begin{figure}[t]
 \centering
 \includegraphics[width=7cm]{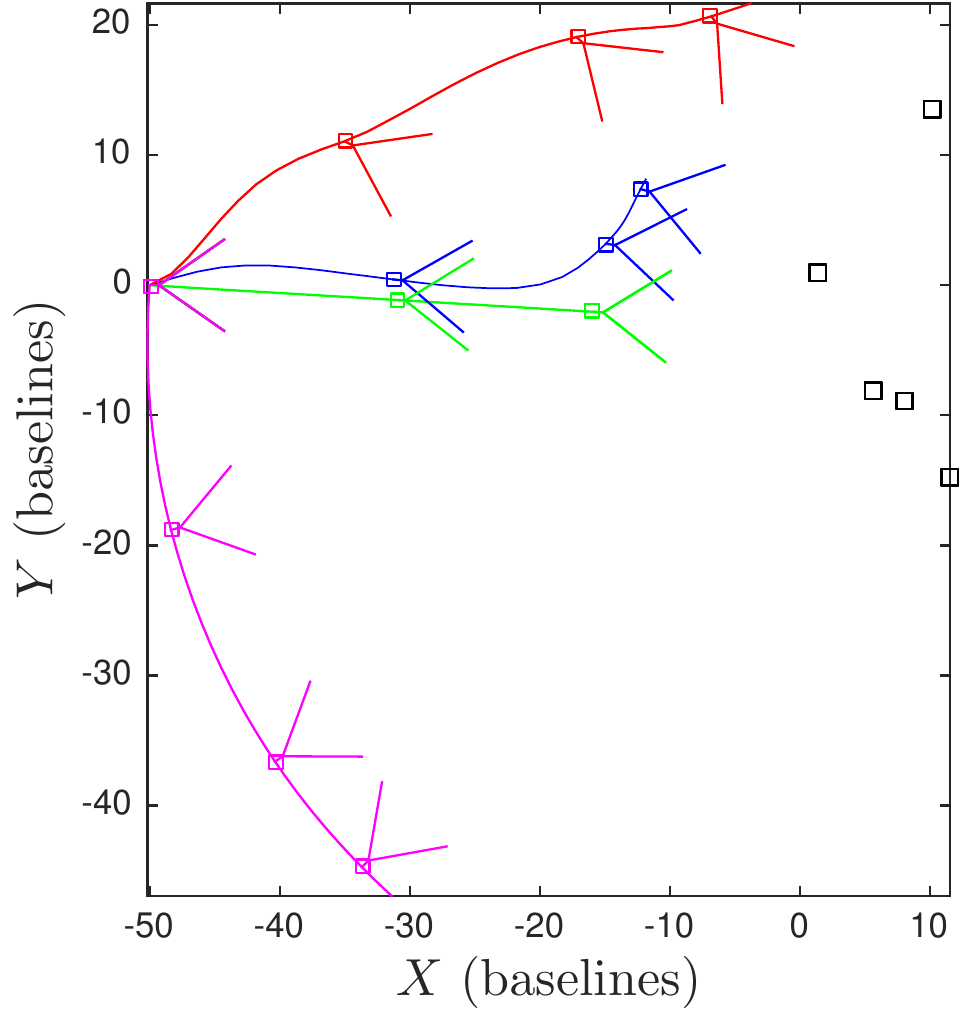}
 \caption{An example of the trajectories generated by our algorithm and the baseline methods, shown in 2D for readability.
Red denotes the trajectory of the supremum, blue denotes the centroid, magenta denotes the circle baseline method, and green the straight baseline method.
The $\square$ symbols show the ground truth target locations.
The triangles emanating from the trajectories represents the orientation and field of view for each objective (see fig.~\ref{fig:fov_3d}).
}\label{fig:traj_static}
\end{figure}

\begin{figure}[t]
 \centering
 \includegraphics[width=7cm]{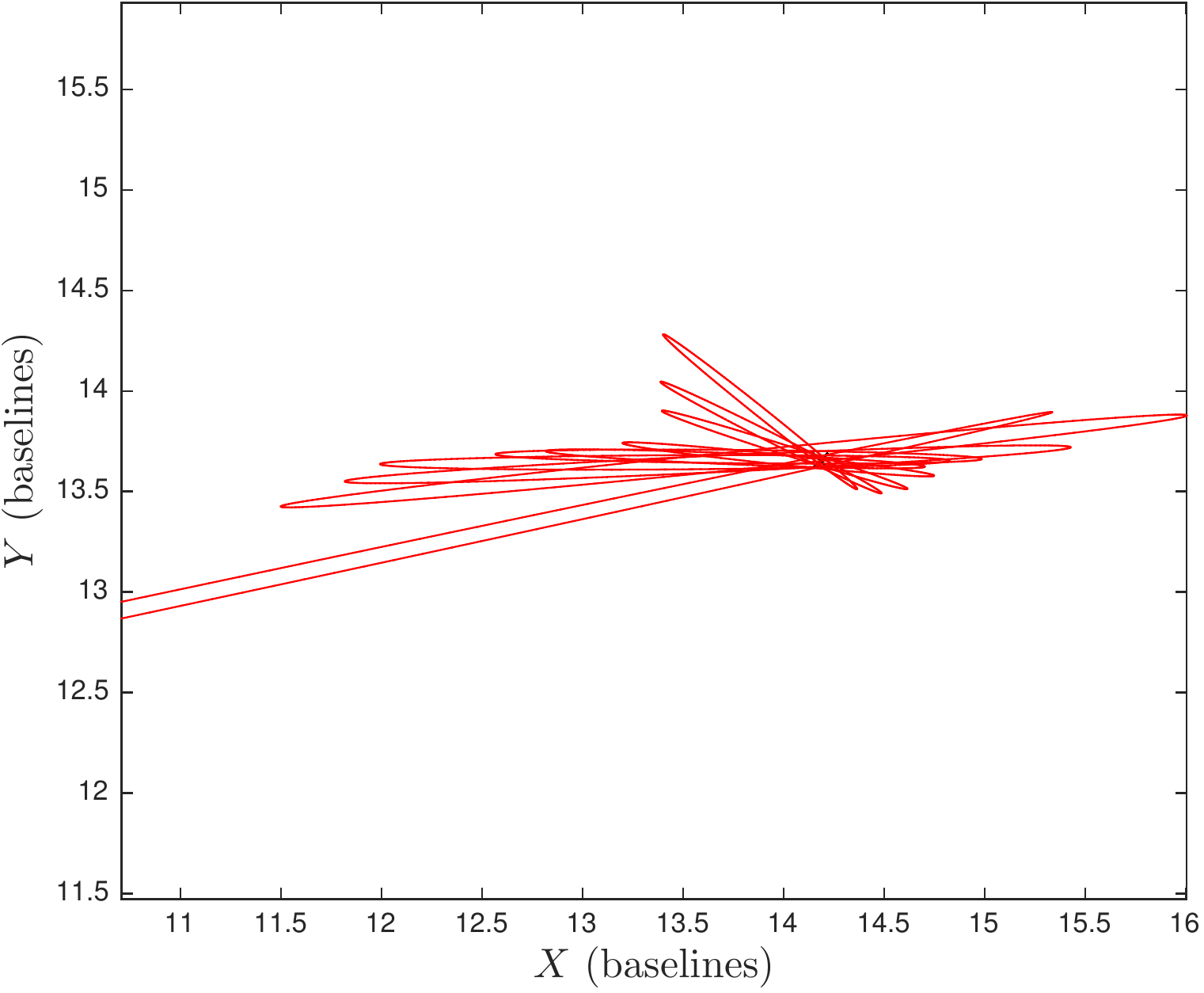}
  \caption{Ten of the $3\sigma$ confidence intervals produced by the supremum objective for one of the targets in Figure~\ref{fig:traj_static}.} \label{fig:ell}
\end{figure}

The following set of simulations considers target localization in three dimensions.  
The goal is to evaluate our algorithm against the baseline methods.
We use an image resolution of 1024$\times$1024 pixels.
The unit of measure is the distance between the two cameras in the stereo rig, or the baseline, which is the characteristic length in stereo vision.
It is depicted as $b$ from Figure~\ref{fig:fov_3d}.
The stereo rig moves 10\% of its baseline between successive images, which corresponds to a $Dt$ in the simulations of $0.1$.
The matrix $Q$ was set to the identity matrix.
In every simulation, the robot begins 50 baselines west of a cluster of targets, which are placed according to a uniform random distribution in the unit cube centered at the origin.  
The penalty parameter $\rho=100$ ensures that all targets remain within the camera's 70$^\circ$ field of view throughout.
The length of the time interval $t_{k+1}-t_k$ between two consecutive observations is chosen so that the robot either realizes the NBV, i.e., achieves $\psi=0$ in \eqref{eq:rRdiff}, or the robot travels the maximum allowed distance between observations.
The gain parameter, $K$ from \eqref{eqn_p_update}, is set to $K = \diag (1,1,7).$
The observers that follow the circle baseline method and straight baseline method at each iteration travel a distance equal to the maximum of the distances that the supremum and centroid traveled in that iteration.
All motion plans make the same amount of total observations.
All use identical camera parameters.  
All observations suffer from quantization noise after pixel coordinates are rounded to the nearest integer. 

\begin{figure}[t]
\centering
\includegraphics[width=7cm]{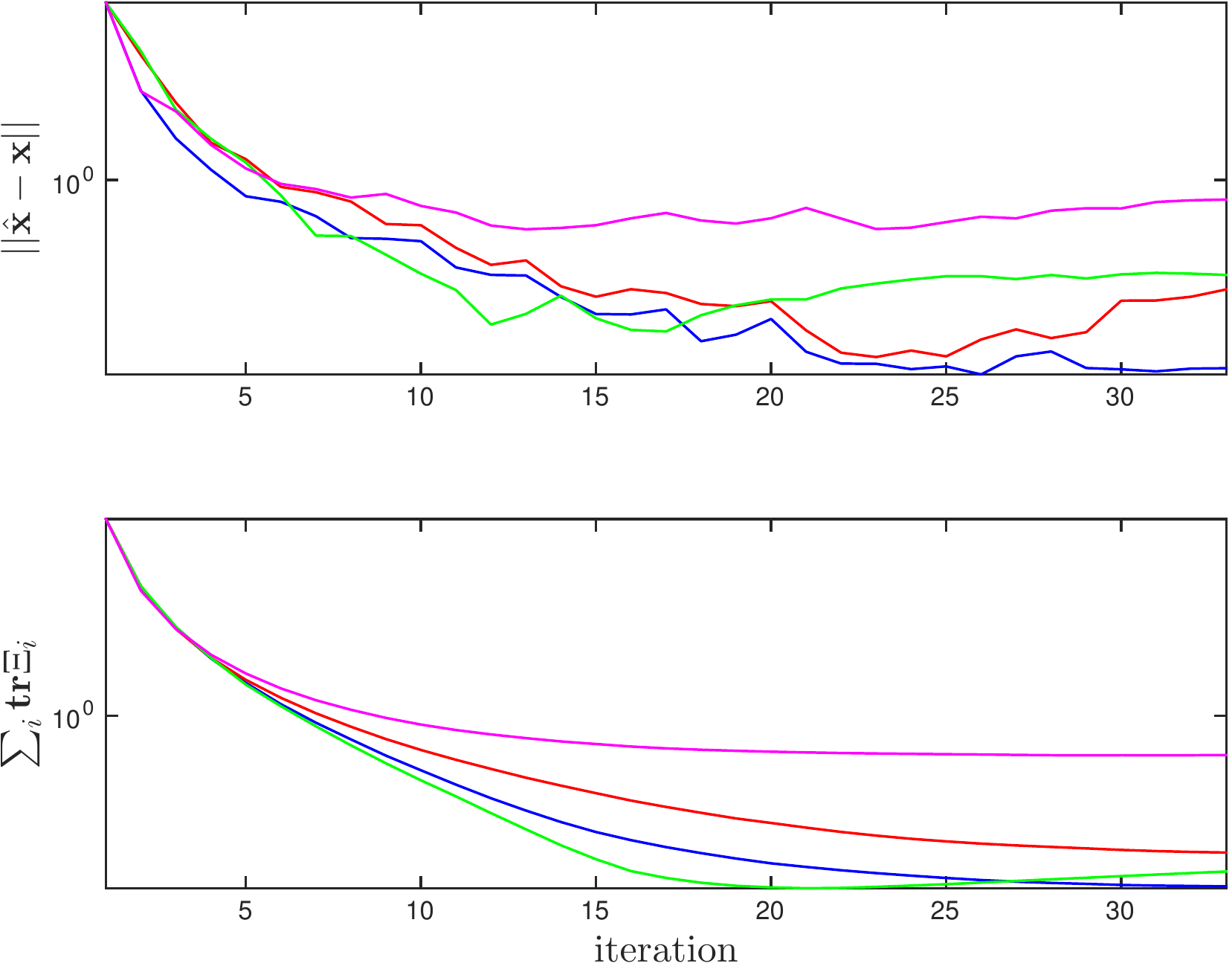}
\caption{Average localization error (top panel) and trace (bottom panel) of the position covaraince of the all targets versus iteration, averaged over 50 simulations.
Red denotes the trajectory of the supremum, blue denotes the centroid, magenta denotes the circle baseline method, and green the straight baseline method.
}\label{fig:errors_and_trace}
\end{figure}

Figure~\ref{fig:errors_and_trace} shows the average total error and the trace of the target location covariance matrices for 50 simulations.
In the bottom panel, evidently the straight baseline method outperforms the supremum and centroid objective in terms of the trace of the posterior covariance matrices, up to the point when it stops being able to move.
This is because the centroid and supremum objectives also obtain control inputs from a penalty function, which repels the robot from views that allow targets near the field of view boundary.
The straight baseline method, on the other hand, can go to the point when one of the targets is on the outer edge of the image, allowing it to get closer.
The centroid and supremum objectives still outperform the straight baseline method in terms of localization error.
Note also that once the stereo rig following the straight baseline method stops moving, it suffers from the same quantized noise in every observation, which is biased, and causes the KF to diverge.
The KFs from the rigs following the circle baseline and the centroid and supremum objectives do not diverge because the individual measurement bias changes when the relative vector changes, effectively de-correlating the errors.

\subsection{Mobile Target Localization}

\begin{figure}[t]
 \centering
 \includegraphics[width=7cm]{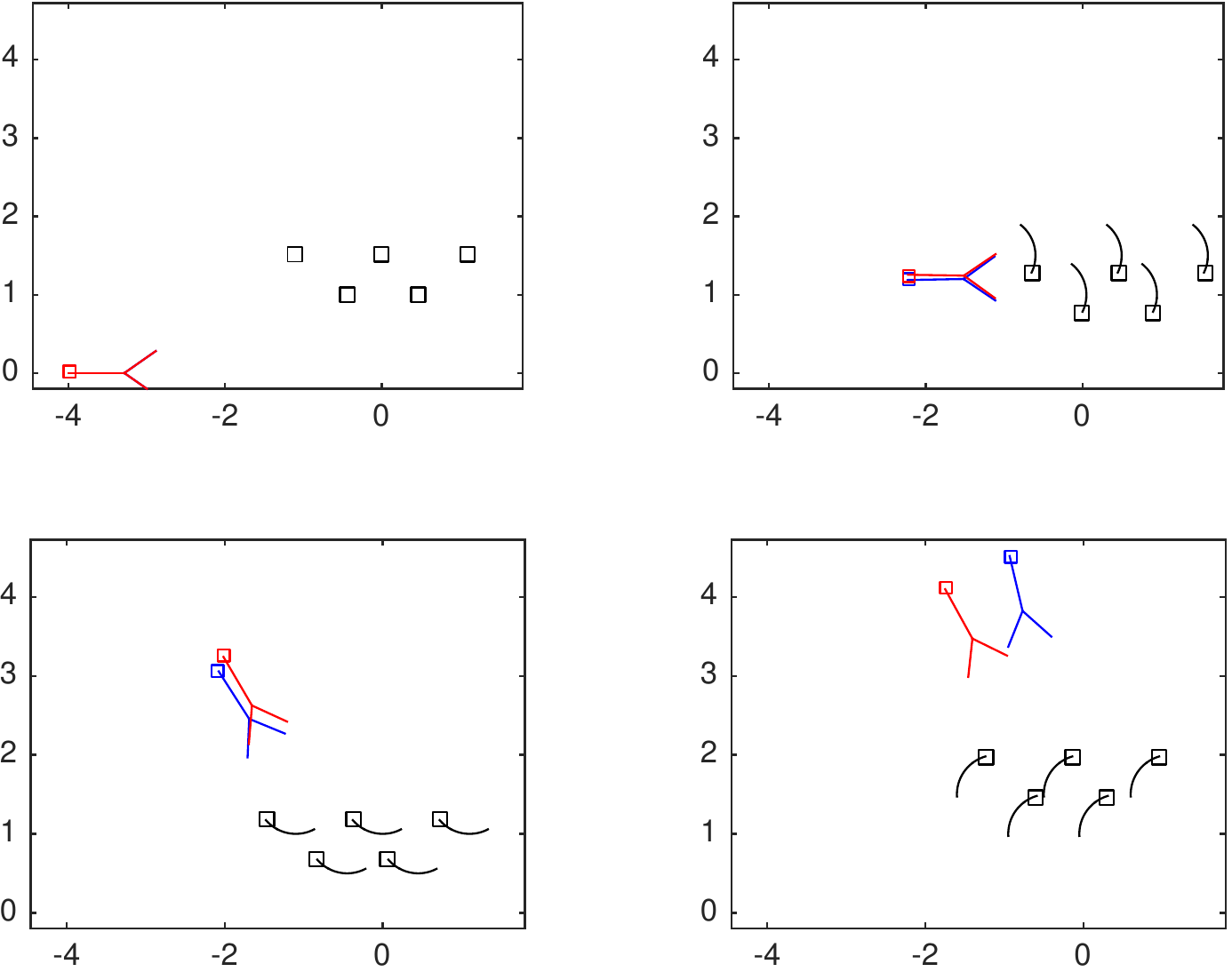}
 \caption{
An example of the trajectories generated by our algorithm when the targets are mobile in two dimensions, with time processing in clockwise order.
Red denotes the supremum and blue denotes the centroid.
The $\square$ symbols show the ground truth target locations, with tails to show their motion.
The triangles emanating from the trajectories represents the orientation and field of view for each objective (see Figure~\ref{fig:fov_3d}).
All units are in baselines.
}\label{fig:traj}
\end{figure}

\begin{figure}[t]
 \centering
 \includegraphics[width=7cm]{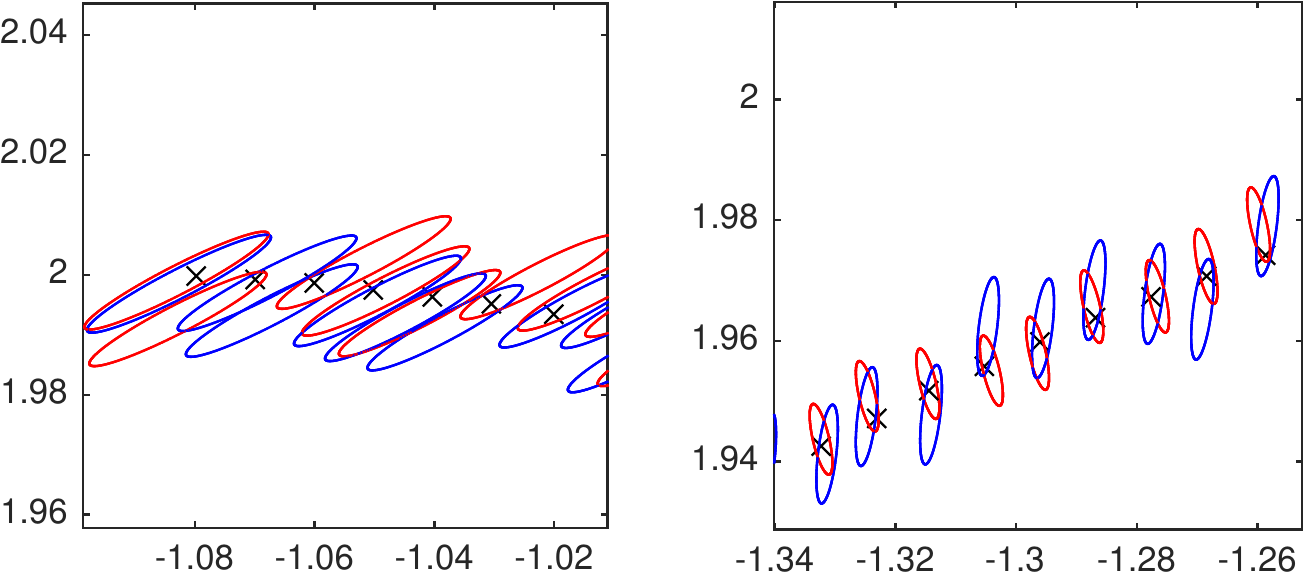}
 \caption{A closeup of the beginning (left panel) and end (right panel) of the left-most target trajectory in Figure~\ref{fig:traj}.  95\% Confidence ellipses associated with each objective are plotted.
 Red denotes the supremum's confidence ellipses, and blue denotes the centroid.
 }
 \label{fig:mobile_ell}
 \end{figure}

\begin{figure}[t]
 \centering
 \includegraphics[width=7cm]{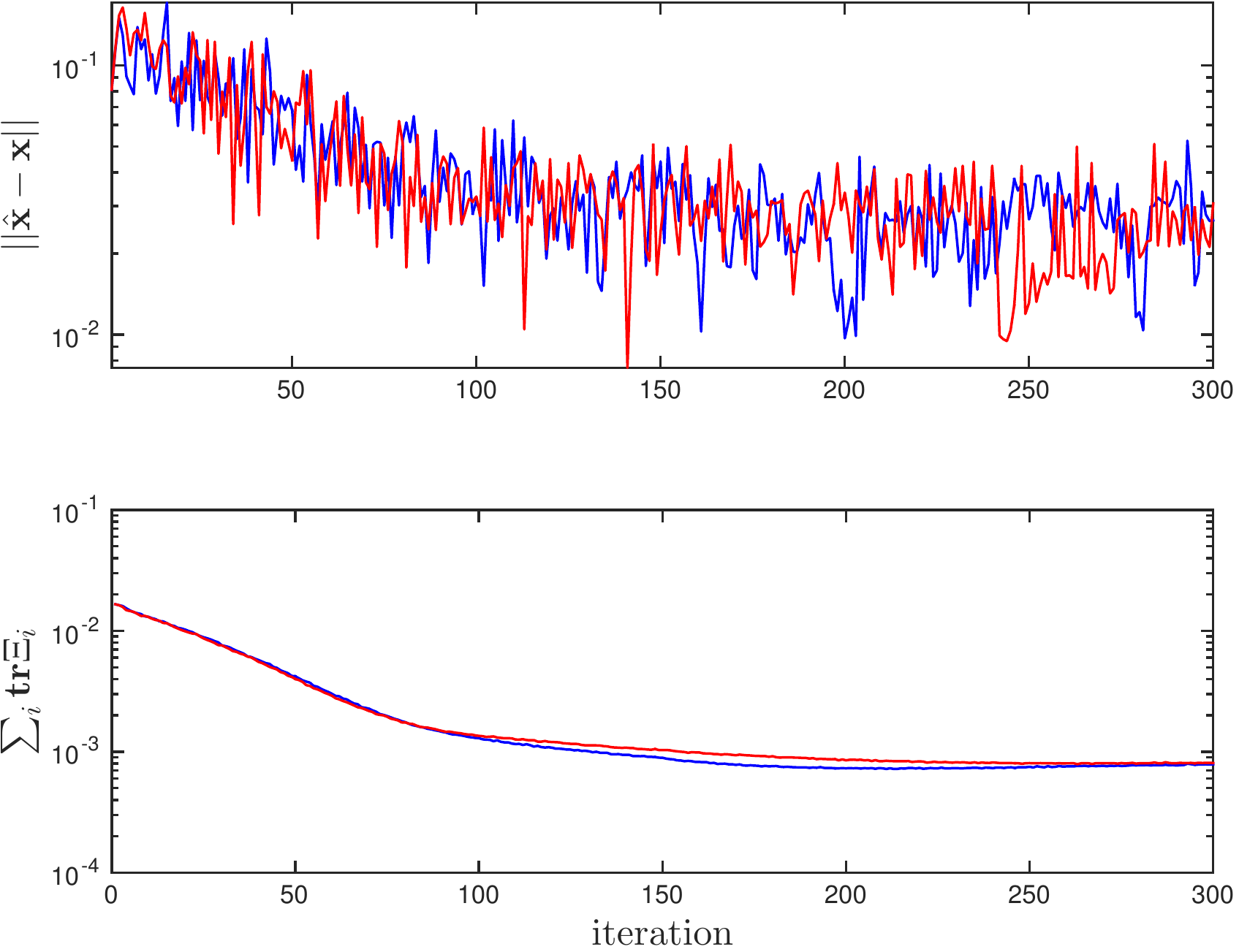}
 \caption{
Localization error (top panel) and trace (bottom panel) of the position covariance of the all targets versus iteration for the mobile simulation shown in Figure~\ref{fig:traj}.
 }
 \label{fig:err_and_trace_mobile}
 \end{figure}

In these simulations, the mobile stereo camera localizes a group of mobile targets that move in the Olympic ring pattern.  
The observers, two cameras implementing the supremum and centroid objectives, use the constant acceleration model from \eqref{eq:genmod}.  
As a simple example, Figures~\ref{fig:traj} and \ref{fig:mobile_ell} show an example of the mobile target simulation in two dimensions.
We present the results of 50 simulations for the mobile target scenario, again subject to quantized noise from pixelation and again in three dimensions.  
All constants used in the mobile simulations are the same as the static simulations.

The top panel of Figure~\ref{fig:err_and_trace_mobile} shows the average error during the 50 simulations with mobile targets in three dimensions.
Because none of the targets stray far from the rest, the \emph{centroid objective} has a slight advantage over the supremum objective.
We also performed simulations with asymmetric data sets and outliers, which favored the supremum objective.  
Any nondecreasing properties in the top panel of Figure~\ref{fig:err_and_trace_mobile} are due to quantized observations.
The correlation coefficient between the time series representing the target error (top panel of Figure~\ref{fig:err_and_trace_mobile}) and that representing the traces of the covariance matrix sequence (bottom panel of Figure~\ref{fig:err_and_trace_mobile}) is 0.84 for the centroid objective and 0.87 for the supremum objective, showing that these proxies are reasonable for localization accuracy.
We also note that the flattening out of the objective function value, plotted in the bottom panel of Figure~\ref{fig:err_and_trace_mobile}, is due to the process noise covariance \eqref{eq:processnoise} preventing the KF-updated covariance from converging to the zero matrix.
This term prevents the covariance from converging to zero in the mobile target case, and instead holds it near the heuristic value given in \eqref{eq:processnoise}.
Overall localization accuracy could be further improved by \emph{a priori} knowledge of motion model, the on-line adaptive modeling of \cite{Li03_part1}, and using multiple observers, as in \cite{roumeliotis02}.

\section{Experiments}
\label{sec:exp}

\begin{figure}[t]
\centering
\includegraphics[width=0.7\columnwidth]{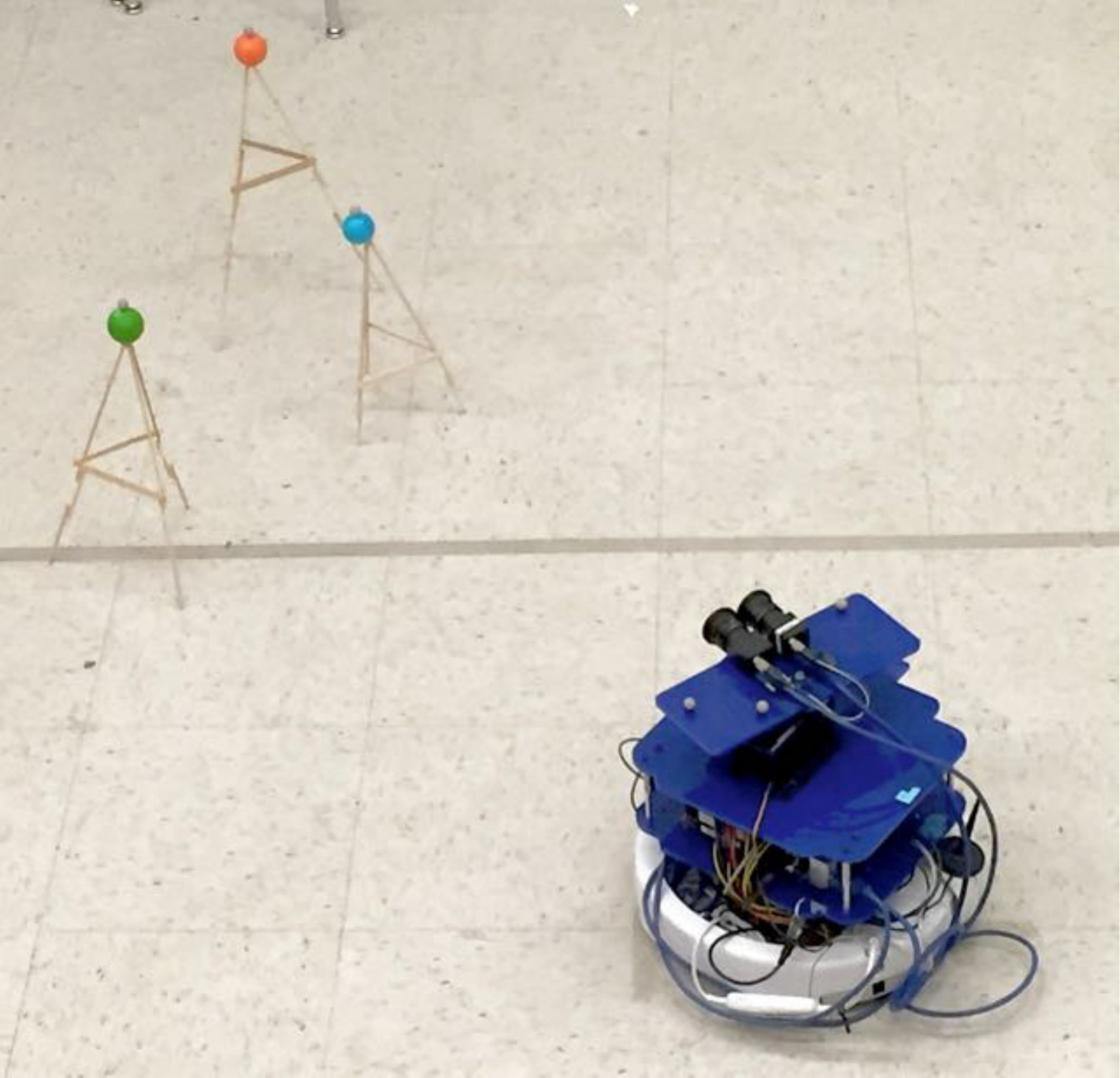}
\caption{Overhead photograph of the experimental setup
}
\label{fig:above}
\end{figure}

In this section, we present experiments using a single ground robot (iRobot Create\texttrademark), pictured in Figure~\ref{fig:above}, to localize a set of stationary targets, for which we used colored ping pong balls.
The robot carries a stereo rig with 4 cm baseline mounted atop a servo that can rotate the rig $\pm 180^\circ.$
The rig uses two Point Grey Flea3\texttrademark\, cameras with resolution $1280\times1024$.
To simulate long distance localization, all images are downsampled by a factor of 24 so that the effective resolution is $54\times43$, allowing us to operate with disparities at or below ten in our laboratory environment. 
The robot is equipped with an on-board computer with 8GB RAM and an Intel Core i5-3450S processor.
All image processing and triangulation is done on-board using C++ and run on Robot Operating System (ROS).
We used the Eigen library for mathematical operations and the OpenCV library for HSV color detection in our controller.

We use the \cite{bouguet2004camera} toolbox to calibrate the intrinsic and extrinsic parameters of the stereo rig offline. 
For self-localization, our laboratory is equipped with an OptiTrack\texttrademark\, array of infrared cameras that tracks reflective markers that are rigidly attached to the robot.
The robot is equipped with an 802.11n wireless network card, which it uses to retrieve its position and orientation by reading a ROS topic that is broadcast over wifi.
To evaluate the localization accuracy of our algorithm, in addition to saving the robot trajectories, we fix markers to the targets, and the motion capture system records their ground truth locations as well.
Finally, note that estimation takes place in three dimensions, whereas the experimental platform is a ground robot confined to the plane.
All navigation and waypoint tracking relies on a PID controller using the next waypoint, defined by the differential flow in \eqref{eq:rRdiff}, as the set point. 
In the experiment, robots generally came within 2 cm of their target waypoints. 
The servo is capable of orienting the stereo camera with accuracy of $\pm 1^\circ$ compared to the global controller.
No collision avoidance, aside from the implicit collision avoidance from the FoV constraints presented in Section~\ref{sec:fov}, is used in the implementation.

\subsection{Noise Modeling}
\label{sec:noise}
In this paper, we have assumed that pixel measurement errors are subject to a known zero mean Normal noise distribution with covariance $Q$.
The goal of this subsection is to ensure that this assumption is satisfied in practice.
In particular, we use training data to remove average bias in the pixel estimates and estimate $Q$ for our experimental setup.
This is critical for a variety of reasons:
\begin{itemize}
\item If the mean of the pixel measurements is biased, then the KF will not converge to the ground truth.
\item If $Q$ is an under-approximation to the actual covariance of random errors at the pixel level, then the KF will become inconsistent and will not converge to the ground truth, if it converges at all.
\item If our choice of $Q$ is too conservative or heuristic, it may not be informative enough to be useful in the decision process at the core of the controller.
\end{itemize}
We also want to test the system in relatively extreme conditions, particularly at long ranges (small disparities), where \cite{freundlich15cvpr} shows that triangulation error distributions are heavy tailed, biased away from zero, and highly asymmetric, which can exacerbate problems caused by calibration errors.

\begin{figure}[t]
\centering
\begin{tabular}{c c}
\includegraphics[height=4cm]{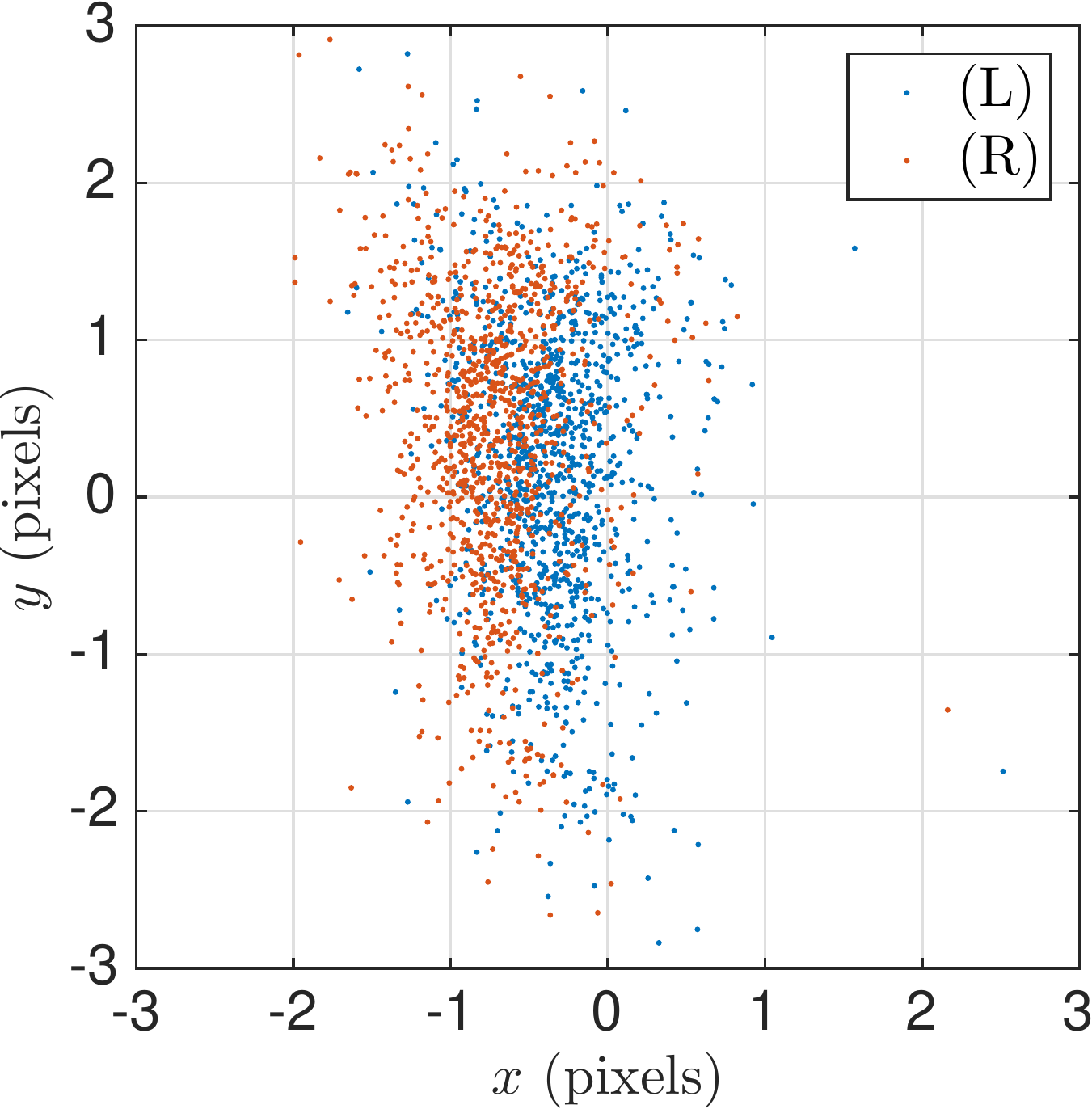}
&
\includegraphics[height=4cm]{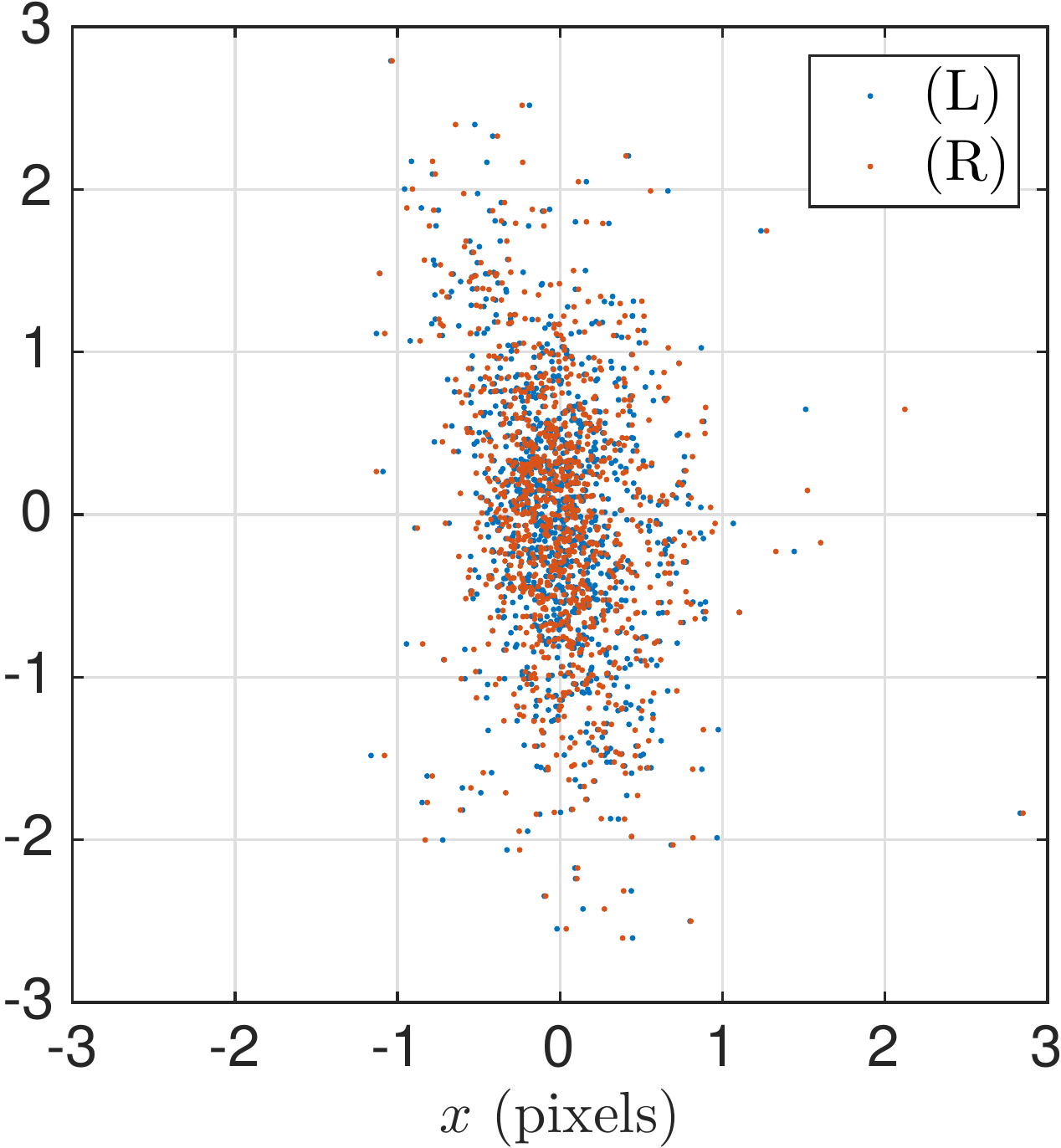}
\end{tabular}
\caption{Scatter plots of the residual errors  $\bbepsilon_\ell^\text{uc}$ (left panel) and $\bbepsilon_\ell$ (right panel) for the training data.
}
\label{fig:px_errors}
\end{figure}

To address these challenges, we adopt a data-driven approach using linear regression in the pixel domain.
Using a set of $n=600$ pairs of training images for the robot at various ranges and viewing angles, we obtain a regression that maps raw pixel observations $(x_L, x_R, y)$ to their best linear unbiased estimate $(x_L^c, x_R^c, y^c)$, hereafter referred to as the \emph{corrected} measurement.
To acquire training data for the regression, we project the motion capture target locations, i.e., the ground truth, onto the camera image sensors, using the mapping in \eqref{eq:invbbp}.
This yields $n$ individual output vectors $Y_\ell$ for $\ell = 1, \dots, n,$ which we stack into an $n \times 3$ matrix of outputs $Y$.
We also use a color detector (the same detector that is used in the experiments) to obtain $n$ raw pixel observations.
We then compute five features and, because the data are not centered, include one constant, for each raw pixel tuple according to the model
\begin{equation}\label{eq:model}
X_\ell = \left[ 
1,\,  y_\ell,\, d_\ell,\,  x_{{L, \ell}}+  x_{{R, \ell}},\,  y  d_\ell, \frac{  x_{{L, \ell}}+  x_{{R, \ell}}}{d_\ell}
\right],
\end{equation} 
where $d_\ell =x_{{L, \ell}} -  x_{{R, \ell}}.$
Stacking the $X_\ell$ into an $n \times 6$ matrix, we have a linear model $Y = X \bbbeta + \bbepsilon,$ where $\bbbeta$ is a $6 \times 3$ matrix of coefficients and $\bbepsilon$ is an $n \times 3$ matrix of errors.
We refer to the raw pixels as \emph{uncorrected}.
The associated error vectors (computed with respect to the uncorrected pixels and the projected ground truth) $\bbepsilon_\ell^\text{uc}$ for $\ell = 1, \dots, n$ are plotted in the left panel of Figure~\ref{fig:px_errors}.
In the scatter plot it can be seen that the mean error is nonzero, contributing average bias to individual measurements.
Also note the apparent skew of the error distribution in the vertical ($y$) direction.

Using the model with the feature vector described in \eqref{eq:model} and applying the ordinary least squares estimator, the maximum likelihood estimate of the coefficient matrix is $\hat{\bbbeta} = (X^\top X)^{-1} X^\top Y.$
Using $\hat{\bbbeta},$ the residual covariance in the pixel measurements we obtained is
\agn*{
Q= \mat{ 
0.1297 &0.1267 &-0.0882 \\
0.1267 &0.1355 &-0.0819\\ 
-0.0882 &-0.0819 &0.6988 
 }.
    }
Note that the standard deviation of the $y$ pixel value, corresponding to the variances in the lower right entry of the above matrix, corresponds to errors in the height of the ping pong ball center in vertical world-coordinates.
The right panel of Fig.~\ref{fig:px_errors} shows the residual errors in the training set $\bbepsilon_\ell$ for $\ell = 1, \dots, n$ for the corrected vector $X \hat{\bbbeta}.$

\begin{figure}[t]
\centering
\begin{tabular}{c}
\includegraphics[width=8cm]{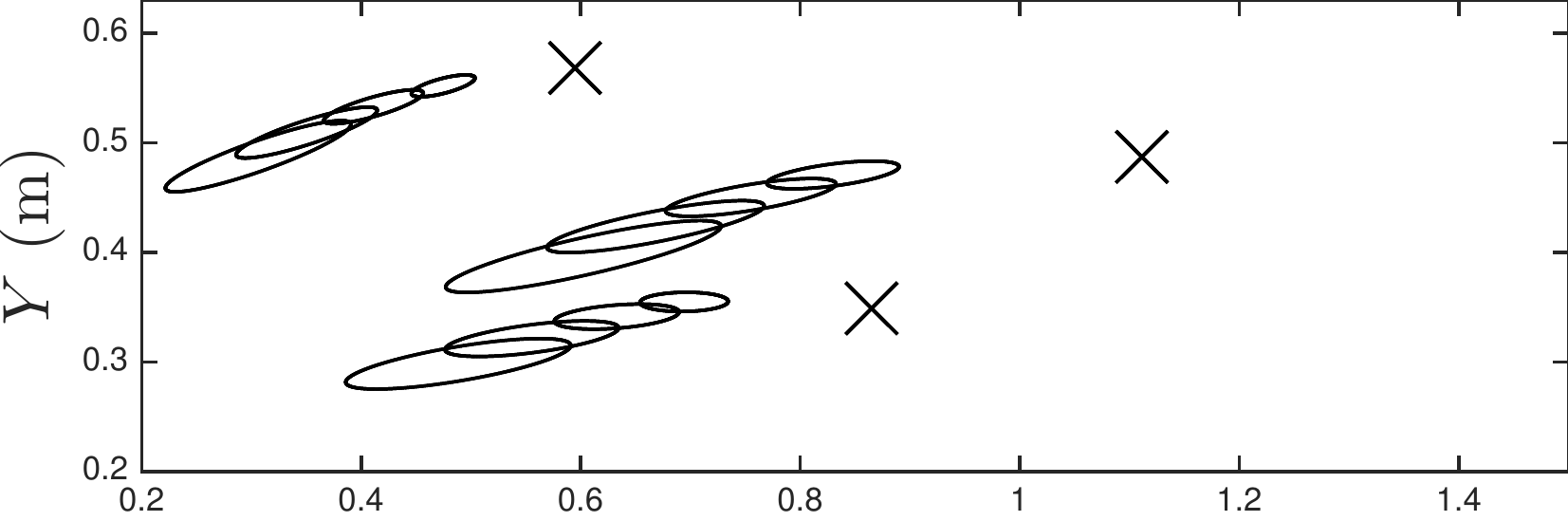}
\\
\includegraphics[width=8cm]{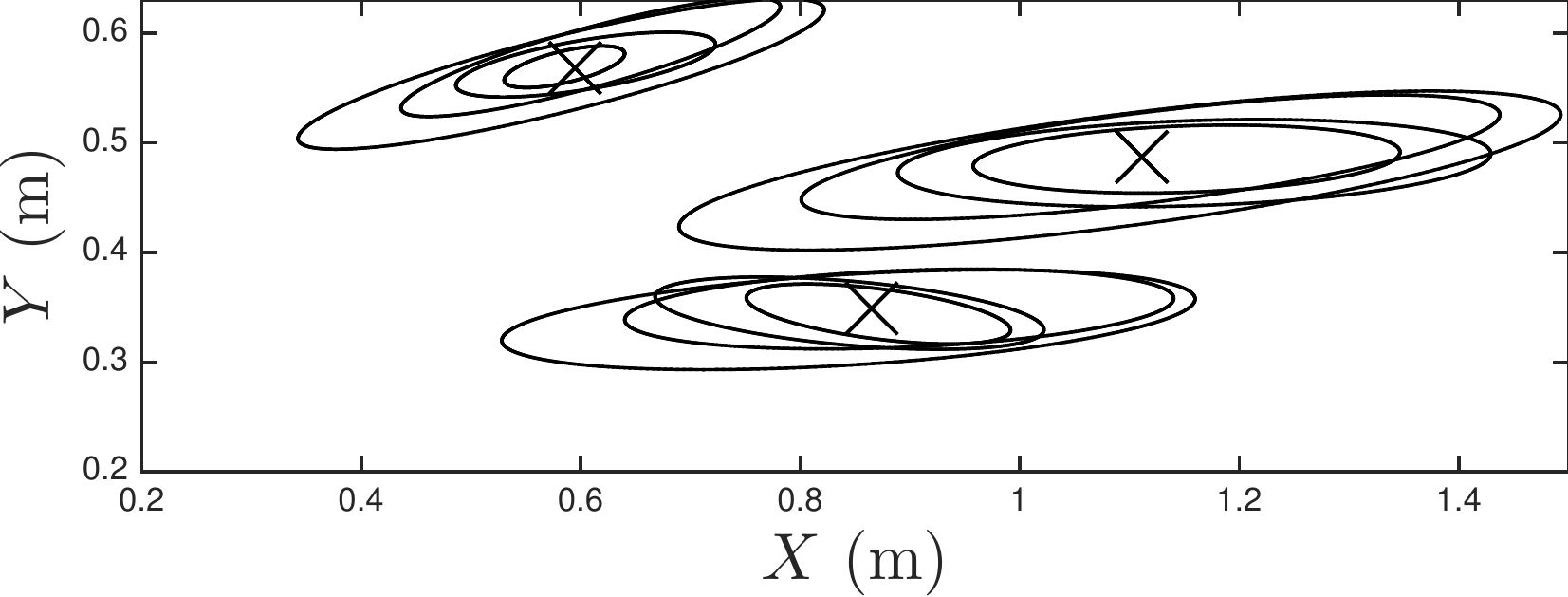}
\end{tabular}
\caption{Projecting the Kalman Filtered 95\% confidence ellipses onto the $X$-$Y$ (ground) plane using the raw/uncorrected (top panel) and corrected (bottom panel) pixel observations.
The $\times$'s denote the true target locations.
The data used to generate these plots were obtained during experimental trials on unseen data.
Projections onto the $X$-$Z$ and $Y$-$Z$ planes gave similar results.}
\label{fig:ellipses}
\end{figure}

To use the learned model online, new raw observations $(x_{L}, x_{R}, y)$ are converted to corrected pixels $(x_L^c, x_R^c, y^c)$ based on the associated new feature vector and $ \hat{\bbbeta}$.
Then, the robot triangulates the relative location of the target via \eqref{eq:bbp} using the corrected pixels, propagates $Q$ via the Jacobian, rotates the covariances, and finally translates the estimates to global coordinates.
Fig.~\ref{fig:ellipses} compares the projection of Kalman Filtered 95\% confidence ellipses onto the $X$-$Y$ (ground) plane using the raw/uncorrected and corrected pixel observations on data that was acquired during the experimental trials.
To generate the plot in the top panel, which corresponds to the result if the raw pixels are used, we computed the empirical covariance of the raw residual errors $\bbepsilon_\ell^\text{uc}$ for $\ell = 1, \dots, n$.

\subsection{Results}

\begin{figure}[t]
\centering
\begin{tabular}{c}
\includegraphics[width=0.8\columnwidth]{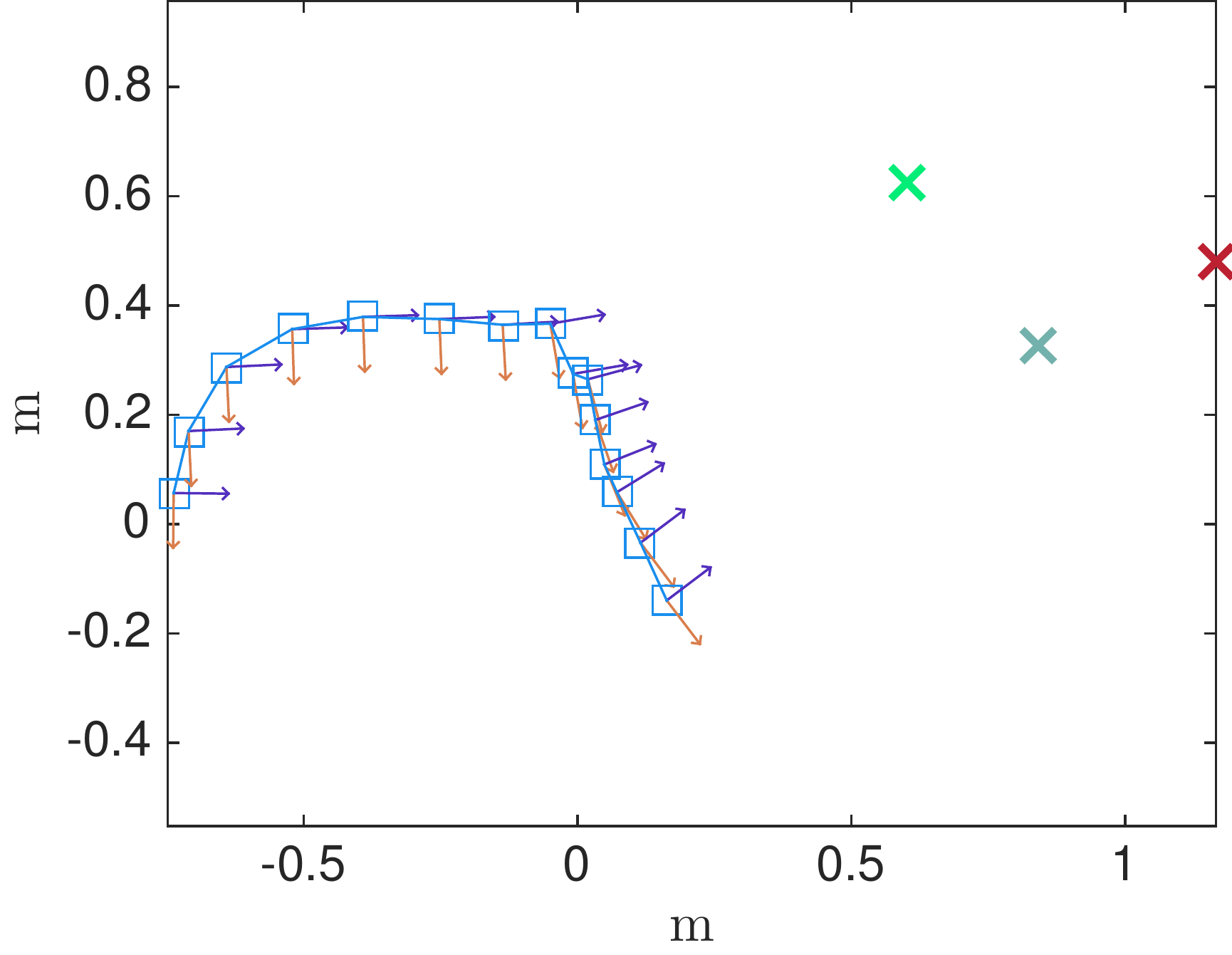} \\
(a) \\
\includegraphics[width=0.8\columnwidth]{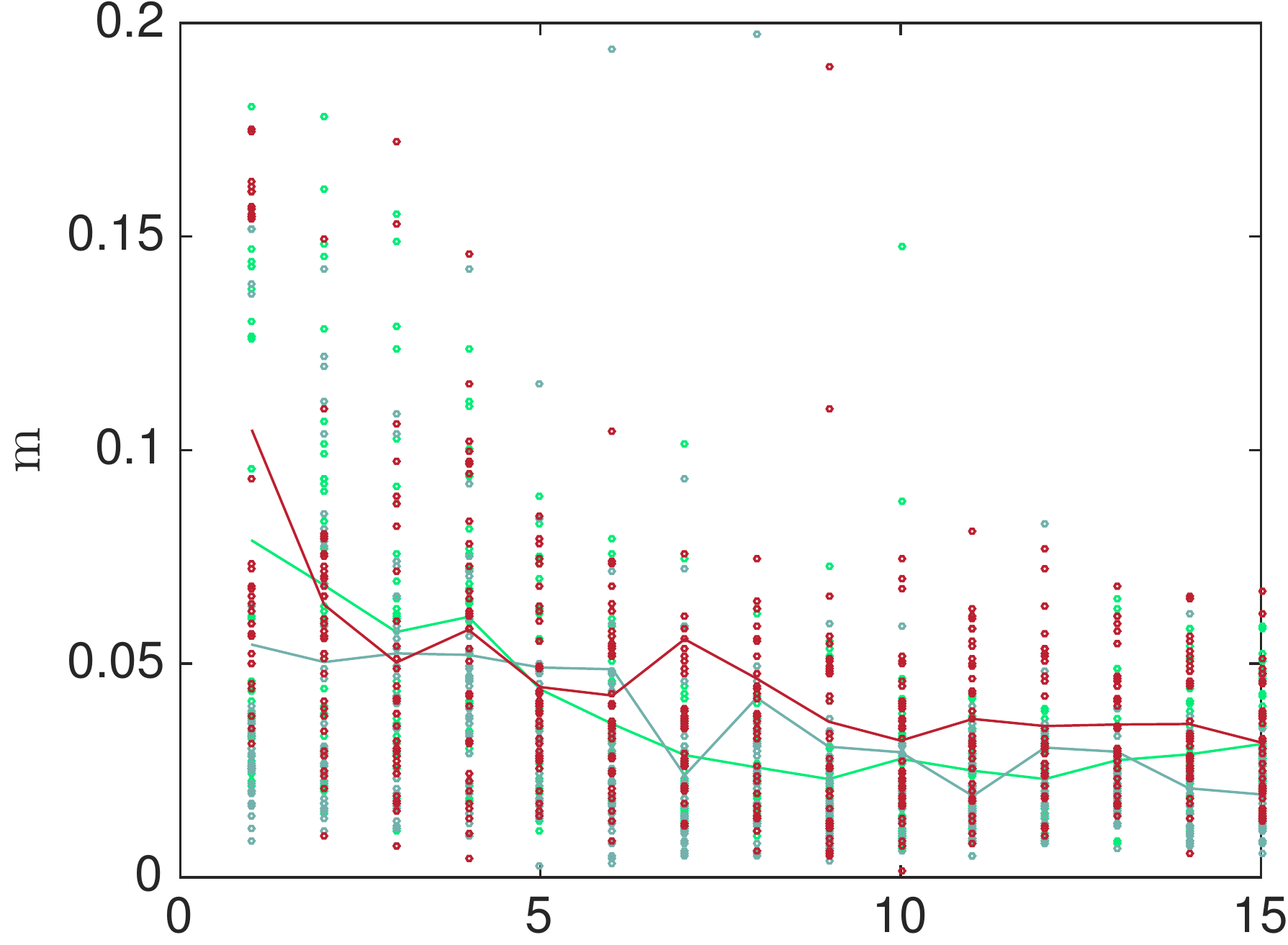} \\
(b) \\
\includegraphics[width=0.8\columnwidth]{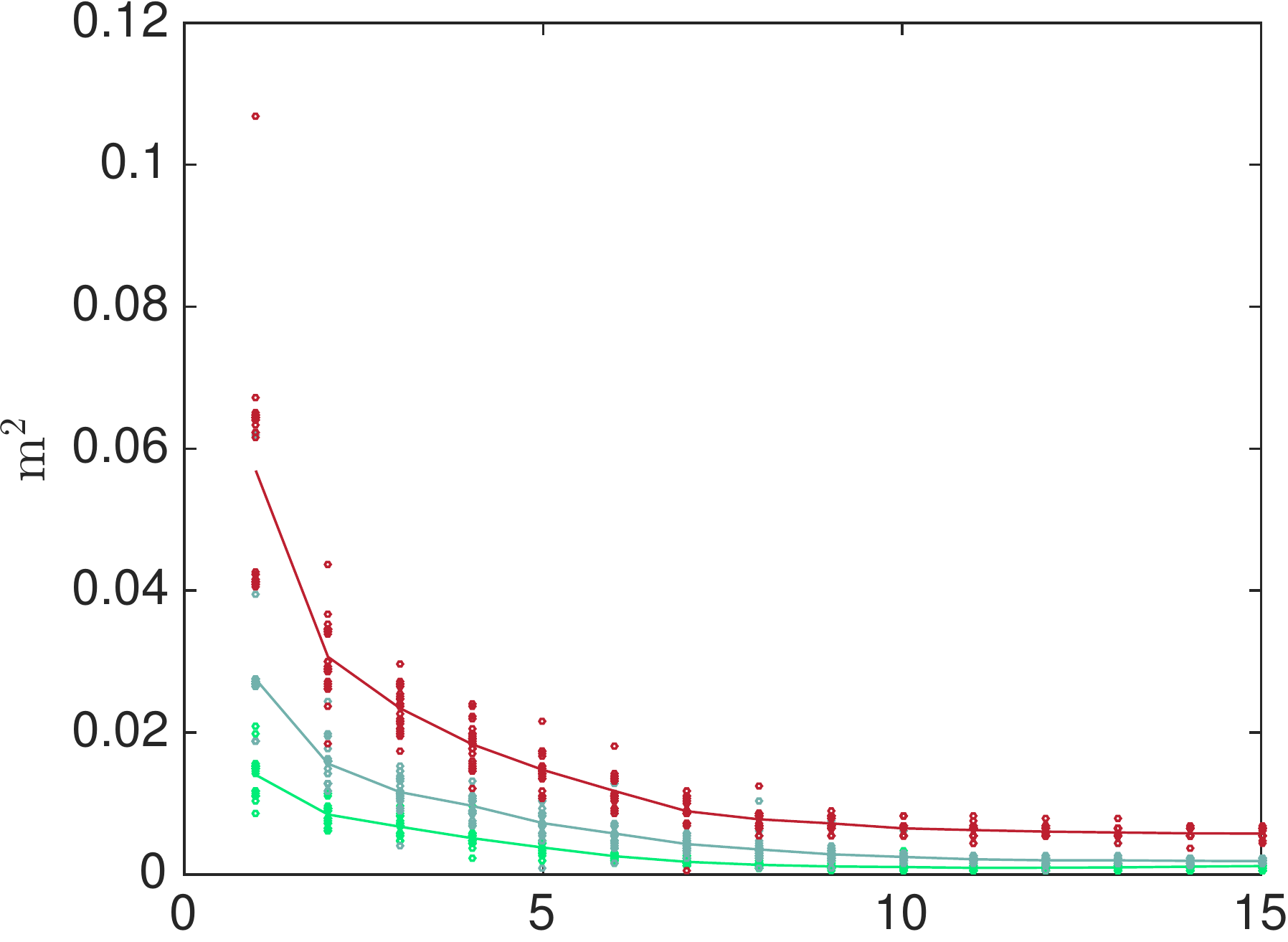} \\
(c)
\end{tabular}
\caption{
Plotting the trajectory of the robot along with the locations of the three static targets for one of the thirty experiments using the supremum objectives.
The blue line represents the position and the $\square$'s are the locations from where an image was taken.
The orientation of the robot (projected onto the plane) is represented at each imaging location by a set of orthogonal axes.
Scatter plots of the filtered error and the trace of the filtered error covariance in all targets for all thirty experiments are also shown.
Each target is plotted using a unique color corresponding to the $\times$.
In the scatter plots, colors correspond to (a), and a line is drawn to guide the eye through the means for each target in the experiment.
The horizontal axes in these plots are the number of images taken.
}
\label{fig:exptraj-sup}
\end{figure}

\begin{figure}[t]
\centering
\begin{tabular}{c}
\includegraphics[width=0.8\columnwidth]{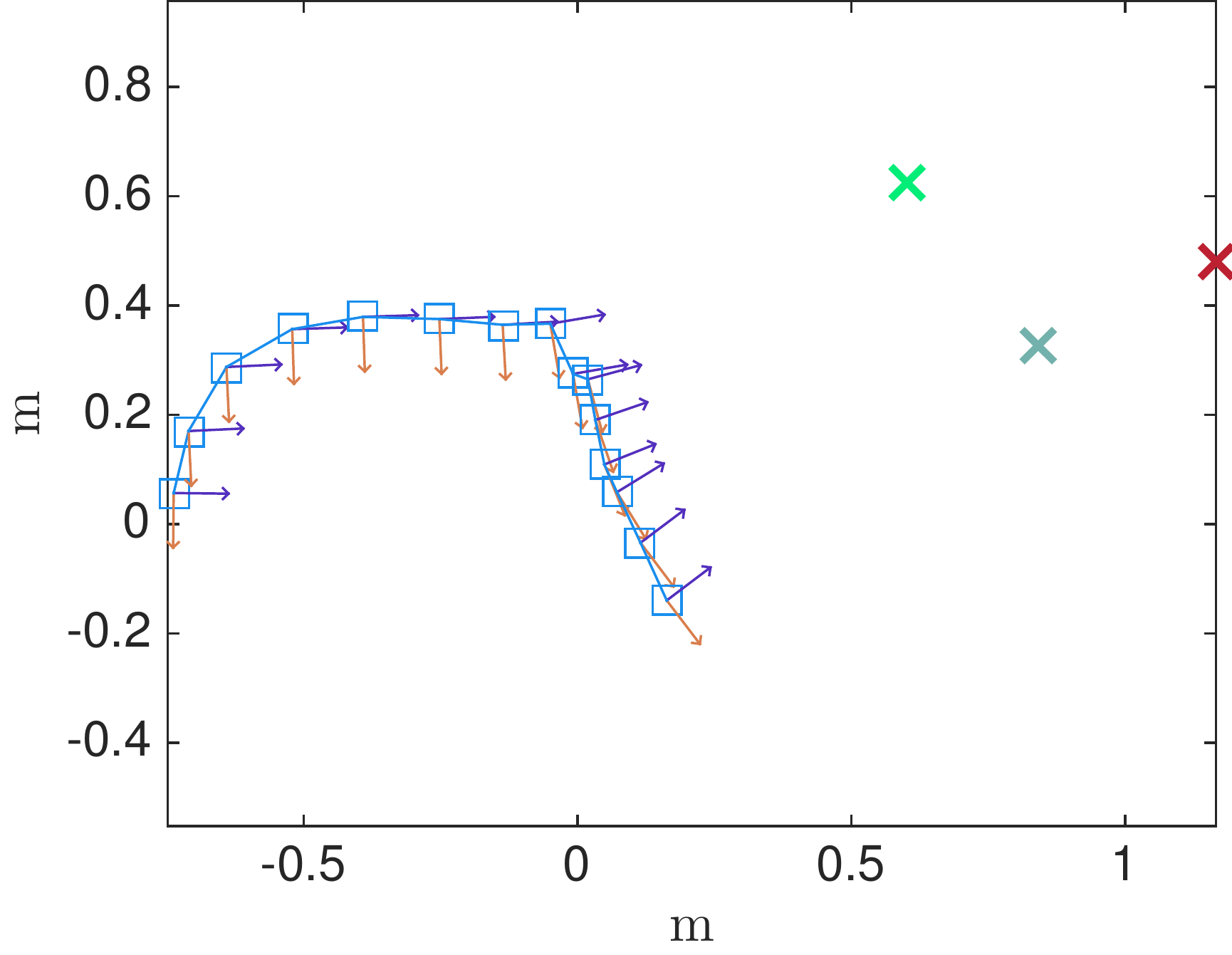} \\
(a) \\
\includegraphics[width=0.8\columnwidth]{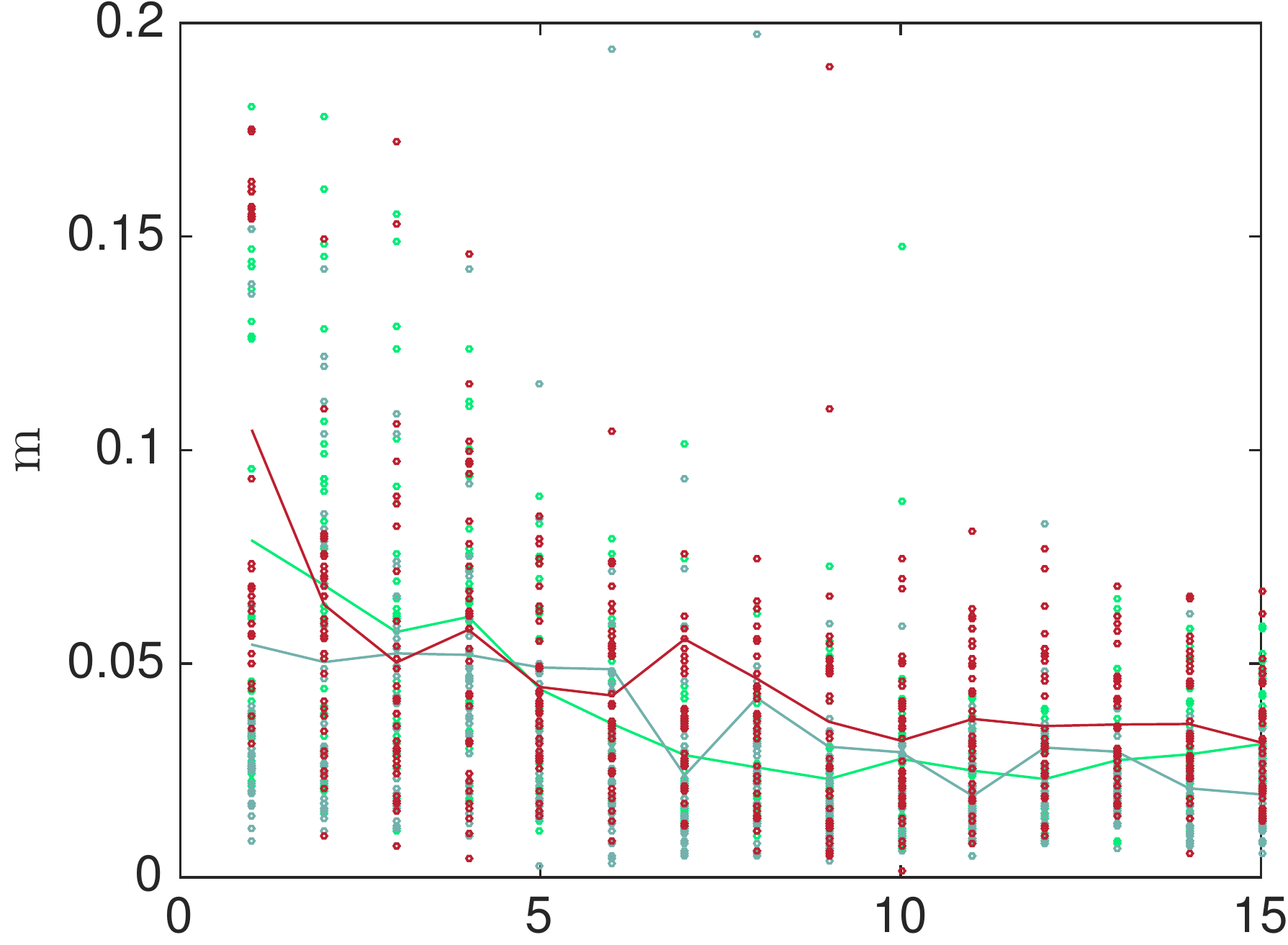} \\
(b) \\
\includegraphics[width=0.8\columnwidth]{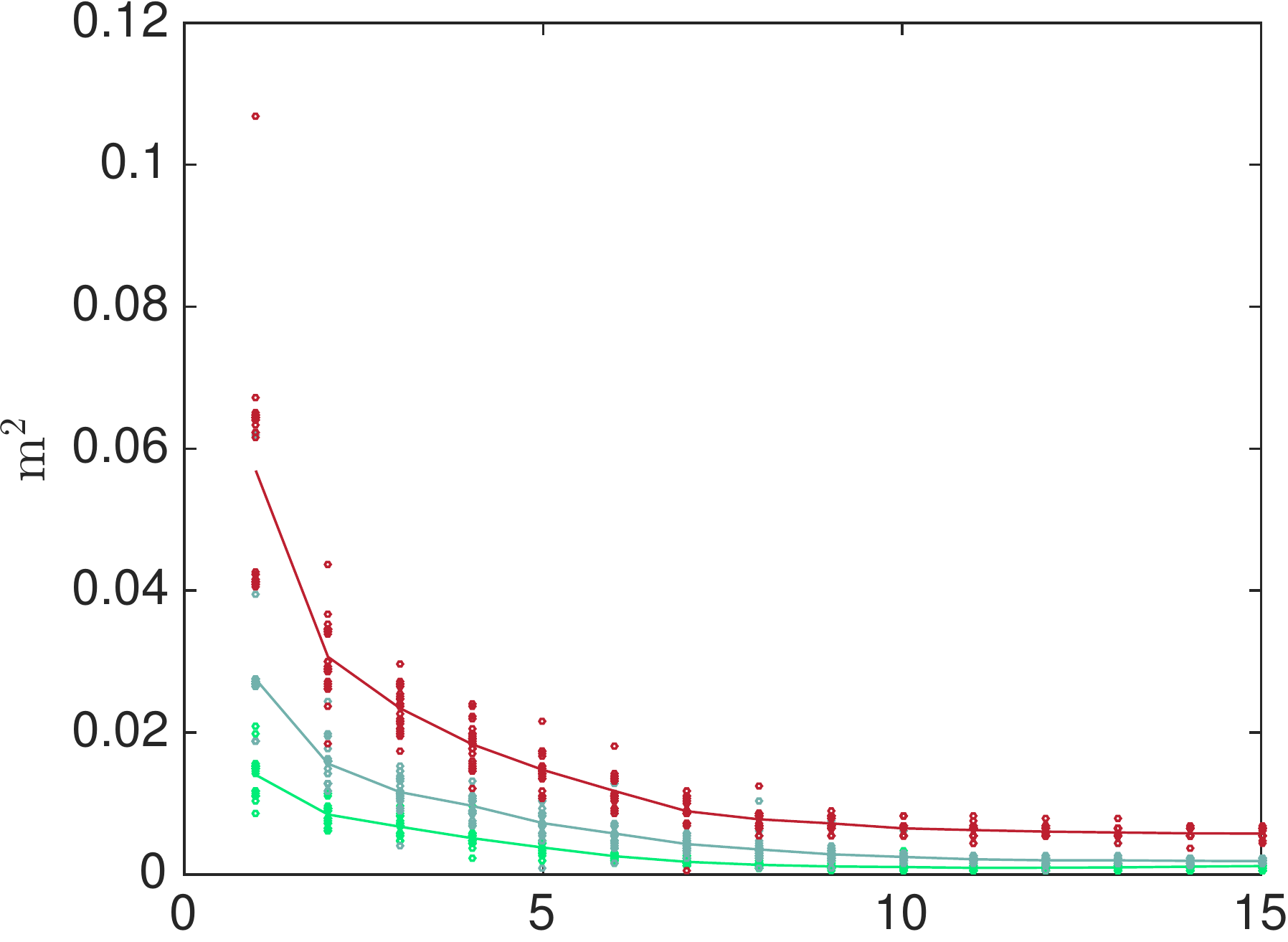} \\
(c)
\end{tabular}
\caption{
Identical plots to Fig. \ref{fig:exptraj-sup}, however the data reported correspond to the centroid objective experiments.
}
\label{fig:exptraj-cen}
\end{figure}

We conducted sixty total static localization experiments -- thirty using the supremum objective and thirty using the centroid objective.
Figures~\ref{fig:exptraj-sup} and \ref{fig:exptraj-cen} (a) show paths followed by the robot during the experimental trials using the setup shown in Figure~\ref{fig:above}.

Figures~\ref{fig:exptraj-sup} and \ref{fig:exptraj-cen} (b) and show scatter plots of the errors from all thirty experiments using each control objective.
Each point in these plots represents the Euclidean distance between filtered estimates and ground truth locations of the ping pong balls provided by the motion capture system.
In each experiment, we collected fifteen images, and in each iteration three targets were observed.
Accordingly, the plots have fifteen bands, each with thirty total points, representing the filtered error in a particular target for a particular experiment.
The mean error for each target across all thirty experiments is drawn on the plot to guide the eye though the scatter plots.

Figures~\ref{fig:exptraj-sup} and \ref{fig:exptraj-cen} (b) reveal the presence of outliers in the localization.
Note from the figure that the KF still converges to ground truth.
We can also see that the overall spread of the bands in the scatter plots is decreasing, reflecting that the control objective is indeed minimized.
On average, the error in each target was reduced by about half, which is less of a reduction than what was observed in simulation.
One reason for this, aside from the presence of unmodeled noise, is the fact that our lab has only about four square meters of usable area, so the diversity of viewpoints is not as rich as in the simulations.

Figures~\ref{fig:exptraj-sup} and \ref{fig:exptraj-cen} (c) show the trace of the filtered error covariance for the same data that was used to plot Figures~\ref{fig:exptraj-sup} and \ref{fig:exptraj-cen} (b).
The points in the scatter plots reflect the posterior variance of each target for each simulation, and again the mean over the thirty experiments using each control objective is drawn on the plot to guide the eye.

\subsection{Comparison to existing heuristic methods that use discrete pose space}

\begin{figure*}[t]
\centering
	\begin{tabular}{c c}
		\includegraphics[width=0.6\columnwidth]{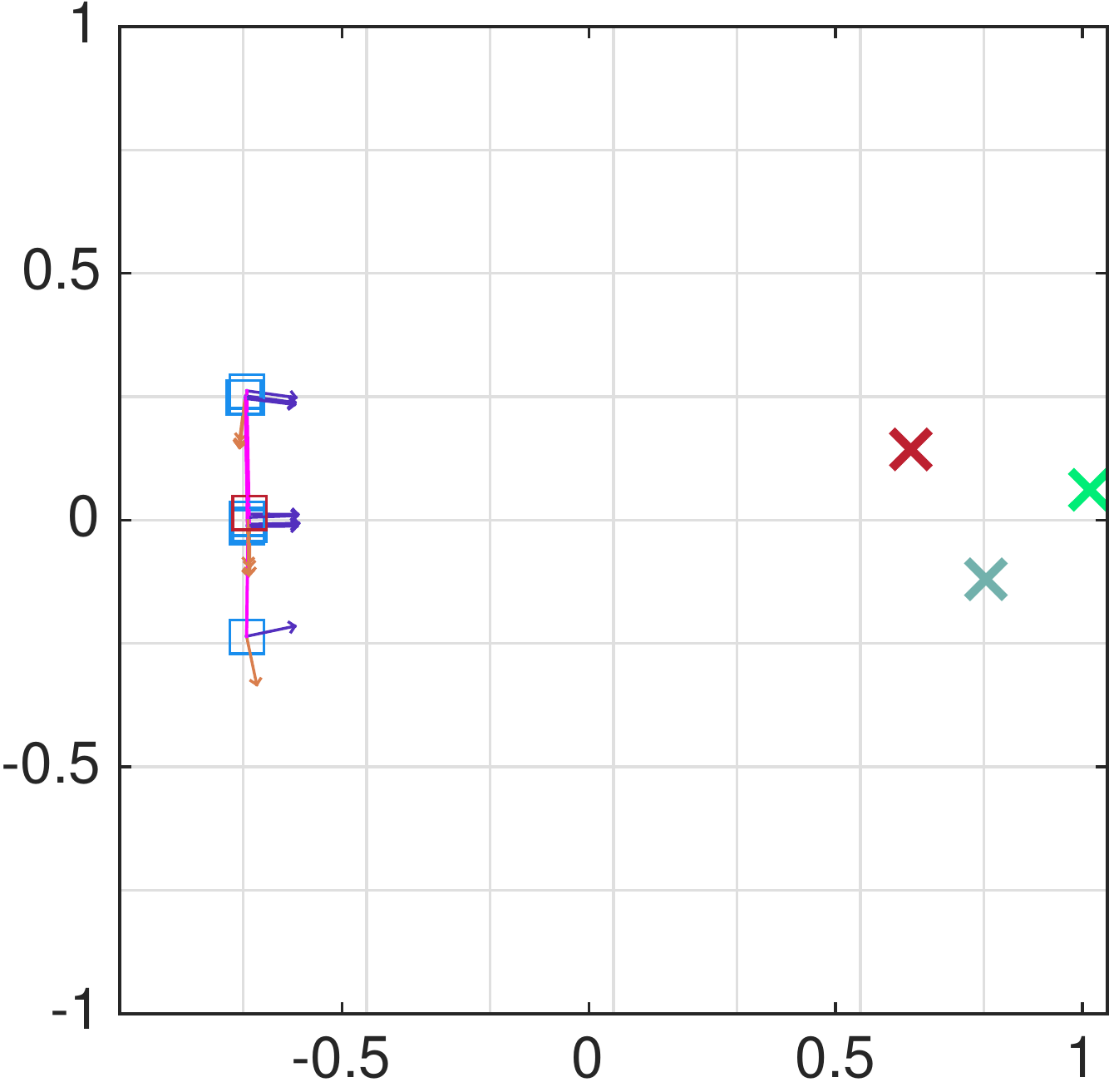} \hspace{1cm} & \hspace{1cm} \includegraphics[width=0.6\columnwidth]{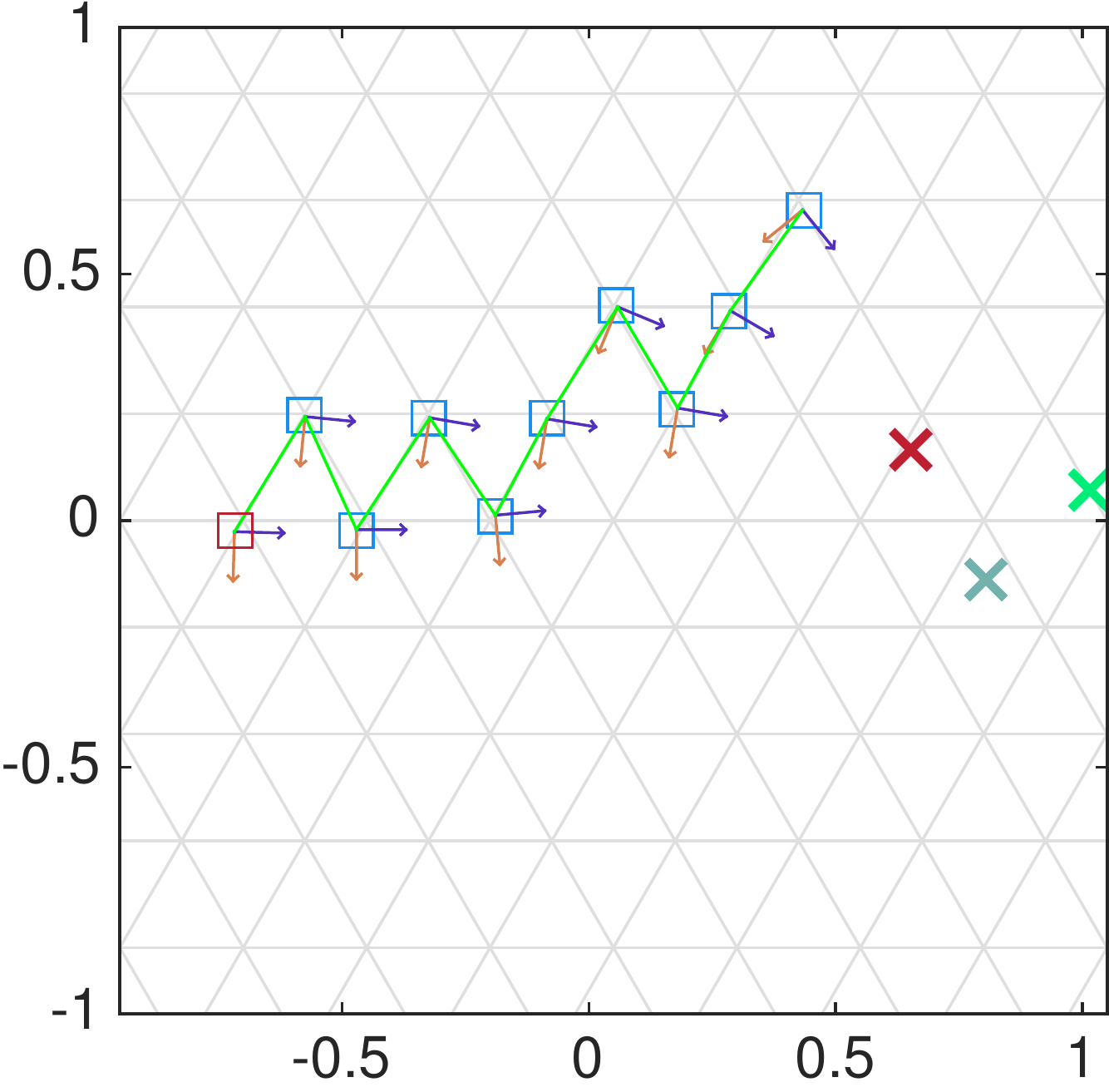}
		\\
		(a) \hspace{1cm}&\hspace{1cm} (b)
		\\
		\includegraphics[width=0.6\columnwidth]{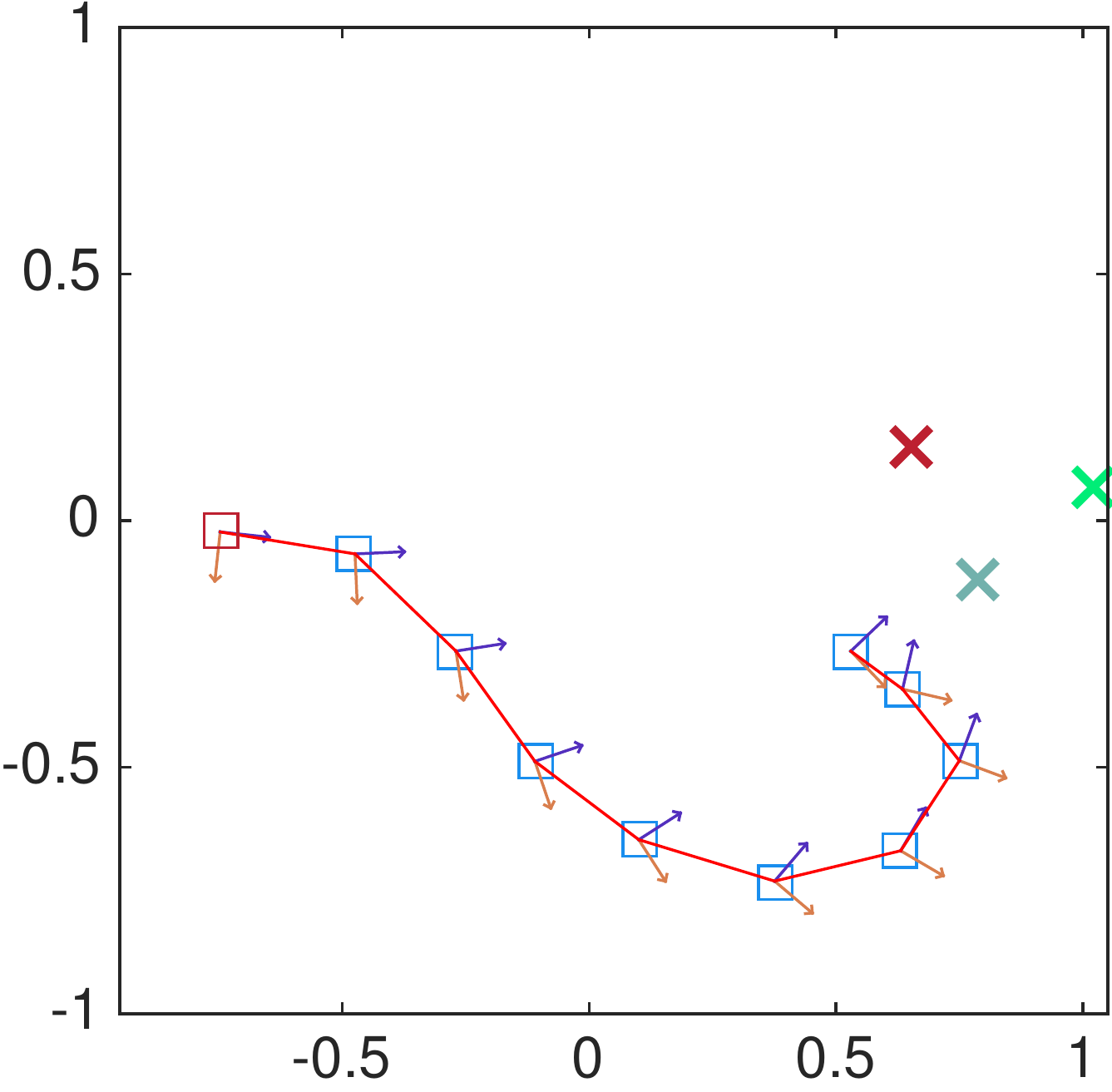} \hspace{1cm} & \hspace{1cm} \includegraphics[width=0.6\columnwidth]{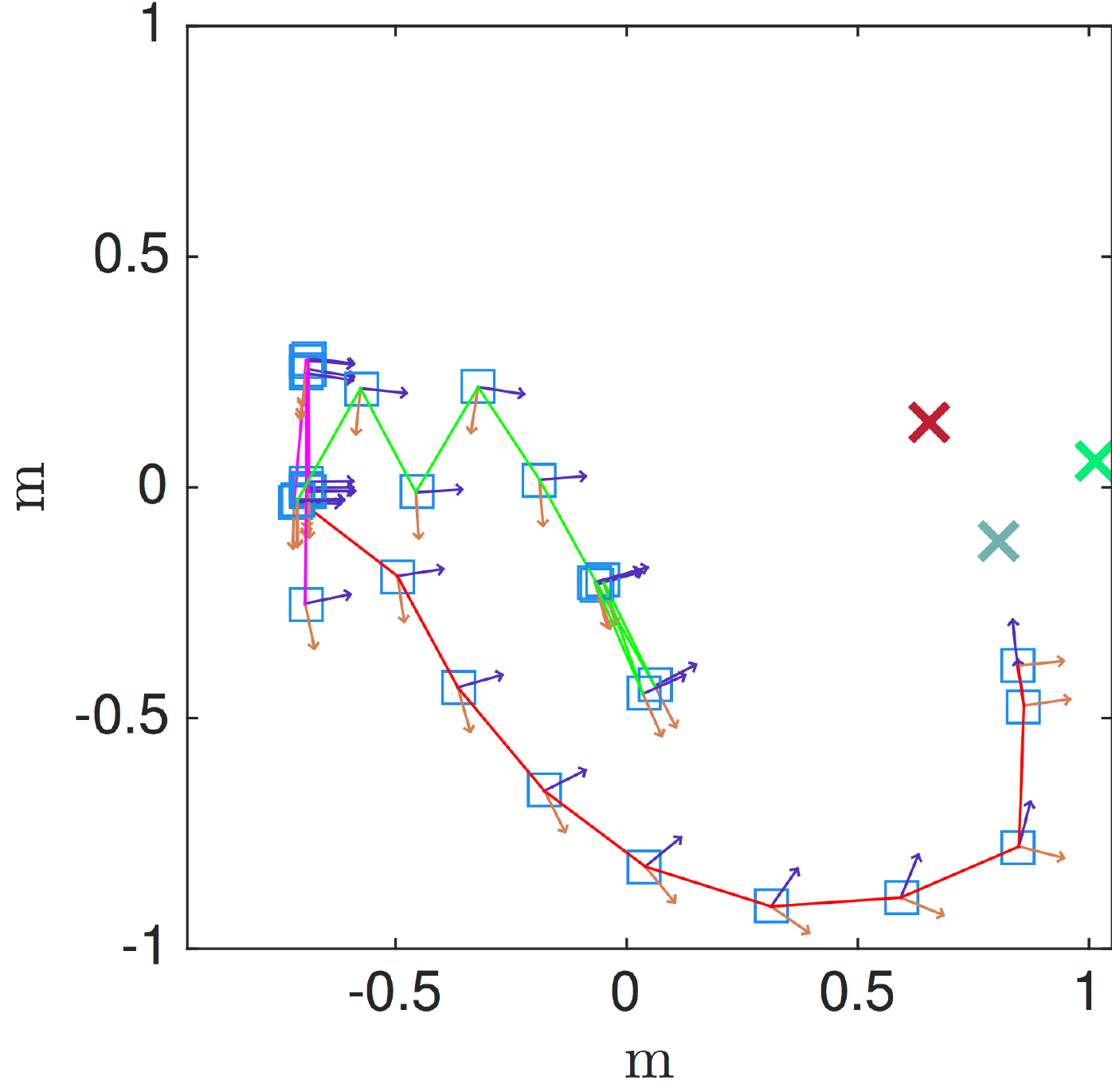}
		\\
		(c) \hspace{1cm}&\hspace{1cm} (d)
	\end{tabular}
\caption{
Exaples of trajectories generated from heuristic comparisons on a 0.25m grid size to our method.
From the top: the heuristic method for the square grid (a), 
triangular grid (b), and
the proposed supremum objective (c). 
Red squares are the initial poses of the stereo rig during the experiments. 
Selected trajectories of the heuristic and the proposed method that contain interesting motion artifacts are shown in (d).
}
\label{fig:exptraj-sup-comp}
\end{figure*}

We conducted static target localization experiments to compare the localization performance of our NBV method to a heuristic method that employs a discretization of the pose space, similar to the approaches discussed in \cite{dunn_iros09, wenhardt07,Wenhardt06}. 
Specifically, the heuristic we implement is based on discretizing the stereo rig's pose space, calculating the objective value in Problem \ref{problem} at all possible next poses and choosing the one having minimum objective value.
This approach is in line with every stereo camera-based approach that we are aware of, in that they all (except \cite{ponda09}, which we have discussed in the introduction) select from discrete next view sets. Note that we cannot fairly compare the localization accuracy of a single camera to a stereo rig (the rig always wins), and we can not compare a stereo rig to a LiDAR system (the LiDAR always wins, assuming that data associations can be established).

In our experiments we focus on the supremum objective, so that the objective value calculated by the heuristic method is the trace of the filtered covariance matrix of the worst localized target. We tested two different ways of discretizing the pose space, namely, a square grid and a triangular grid, as shown in 
Figures~\ref{fig:exptraj-sup-comp} (a)-(b). 
In all experiments, the robot started from the same pose. Moreover, we required that both our NBV method and the heuristic travel approximately the same amount of distance and take the same number of images. In this way, different trajectories can be compared in terms of their ability to localize the targets. This requirement also specifies the edge length for the square and triangular grids.
In this experiment, we set the total number of images that each method can take equal to ten and the edges of the square and the equilateral triangle cells were both set to be 0.25m. At each node of the grid, the stereo rig is oriented towards the estimated position of the worst localized target. This is also the behavior achieved by the NBV method with the supremum objective. We ran each method twenty times and below present our results.

\begin{figure}[t]
\centering
	\begin{tabular}{c}
	\includegraphics[width=0.7\columnwidth]{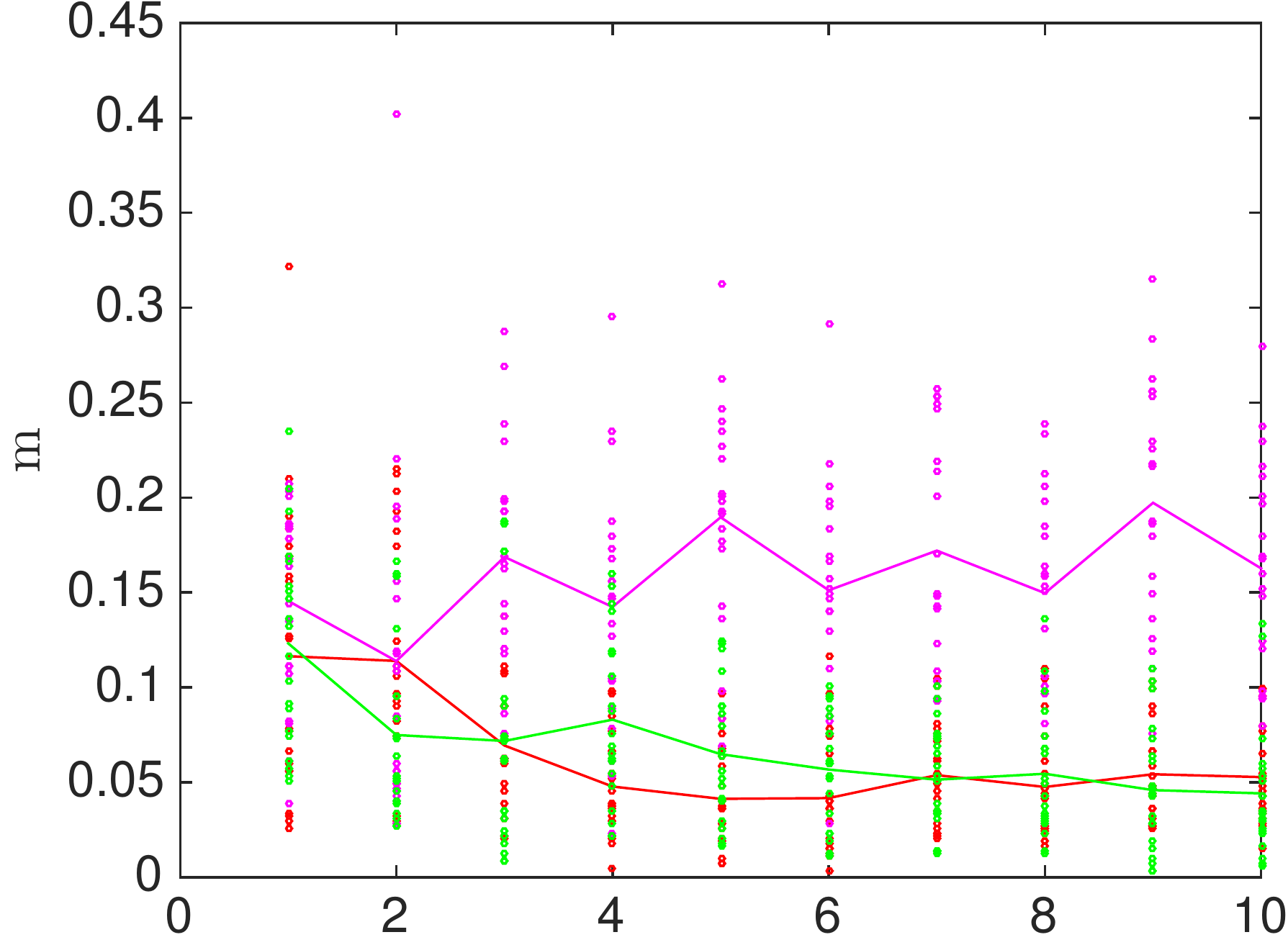} \\
	(a) \\
	\includegraphics[width=0.7\columnwidth]{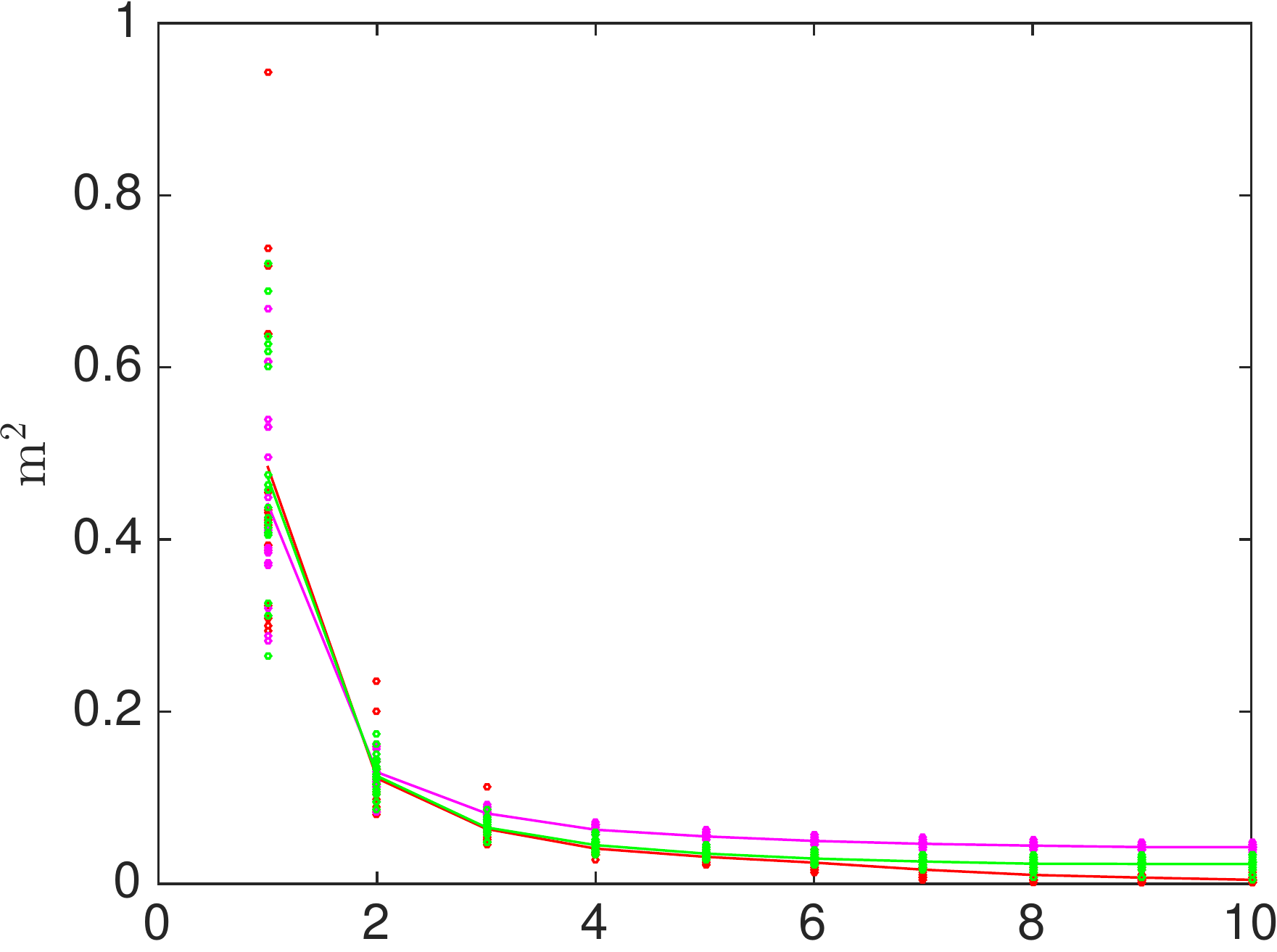} \\
	(b)
	\end{tabular}
\caption{
Filtered localization error (a) and the trace of the error covariance (b) of the green target from Fig.~\ref{fig:exptraj-sup-comp}, averaged over twenty trials for each one of the three methods.
Red corresponds to the proposed NBV method, magenta to the heuristic method for the square grid, and green to the triangular grid heuristic.
}
\label{fig:exptraj-sup-comp-results}
\end{figure}

Figure~\ref{fig:exptraj-sup-comp} (a)-(c) shows sample paths followed by the robot using the heuristic method for the {}two different grids and the NBV approach, respectively, for one of the twenty trials.
To take ten images, the heuristic method that uses the square grid travels an average distance of 2.3058m, the heuristic method that uses the triangular grid travels 2.2198m on average, and our proposed method travels 2.1949m on average.
We note that the heuristic method is highly sensitive to the grid size and error covariance matrices during the measuring process, for both types of grids.
Specifically, during twenty trials of experiments with grid size of 0.25m, the heuristic generated trajectories that contain small cycles (back and forth motion among a few cells) for both grid types; see, e.g., Figure \ref{fig:exptraj-sup-comp} (a) and (d).
In fact, we were unable to find a grid size that does not generate such motion artifacts for the square grid heuristic; in every single trial and for every grid size we observed a behavior similar to the one shown in Figure \ref{fig:exptraj-sup-comp} (a). On the other hand, after much trial and error we found that, for the particular set up of targets in our lab, a 0.25m grid size can produce reasonable trajectories for the triangular grid heuristic, as shown in Figure~\ref{fig:exptraj-sup-comp} (b). Nevertheless, this behavior was not consistent, as seen in Figure \ref{fig:exptraj-sup-comp} (d) for the same grid size. 
Our continuous space NBV controller, shown in Figure~\ref{fig:exptraj-sup-comp} (c), selects the next pose in a continuous pose space and automatically balances the strategies between varying viewing angle and approaching targets. 
Figures \ref{fig:exptraj-sup-comp-results} (a)-(b) demonstrate the localization performance of the NBV method compared to the heuristic method. 
Figure \ref{fig:exptraj-sup-comp-results} (a) shows the filtered localization error during each one of the ten iterations (after each image was taken), averaged over the twenty trials.
The localization error of the heuristic method for the square grid (magenta line)
eventually diverges due to measurement bias that causes the KF to diverge. 
This is the result of observing the targets from the same position. 
A similar behavior was also observed for the straight baseline in the simulations; see Section \ref{sec:subsec_stat_target_local}. 
When the grid edge length is chosen as 0.25m, the heuristic method for the triangular grid (green line) achieves similar localization error as our NBV method (red line).
Figure \ref{fig:exptraj-sup-comp-results} (b) shows the trace of filtered error covariance for the heuristic and the NBV method, averaged over the twenty trials.
In this case, the NBV method outperforms the heuristic for both grid types.
While the heuristic confined to the square performs extremely poorly because it does not approach the targets, the heuristic method on the triangular grid, while slightly better, still does not perform as well as the proposed continuous space method.
Finally, note that the grid size of 0.25m was selected after laborious tuning to remove such artifacts, suggesting that our continuous method will perform better in general situations than the discrete pose space alternatives.

\section{Conclusions}\label{sec_conclusions}
In this paper, we addressed the multi-target, single-sensor problem employing the most realistic sensor model in the literature. 
Our approach relies on a novel control decomposition in the relative camera frame and the global frame. 
In the relative frame, we modeled quantization noise and did not operate under a Gaussian noise assumption at the pixel level (as range/bearing models assume).
Our approach avoids setting covariance by deriving $\Sigma$ from the uniform distribution. 
This allows us to obtain the Next Best View from where the targets can be observed in order to minimize their localization uncertainty.  
We obtain this NBV using gradient descent on appropriately defined potentials, without sampling the pose space or having to select from a set of previously recorded image pairs.
Compared to previous gradient-based approaches, our integrated hybrid system is more precise since it derives Gaussian parameters from the quantization noise in the images.
Furthermore, our approach does not assume omnidirectional sensors, but instead imposes field of view constraints.

\bibliographystyle{IEEEtran}
\bibliography{charlie-refs}

\newcommand{\noop}[1]{}
\begin{thebibliography}{10}
\providecommand{\url}[1]{#1}
\csname url@rmstyle\endcsname
\providecommand{\newblock}{\relax}
\providecommand{\bibinfo}[2]{#2}
\providecommand\BIBentrySTDinterwordspacing{\spaceskip=0pt\relax}
\providecommand\BIBentryALTinterwordstretchfactor{4}
\providecommand\BIBentryALTinterwordspacing{\spaceskip=\fontdimen2\font plus
\BIBentryALTinterwordstretchfactor\fontdimen3\font minus
  \fontdimen4\font\relax}
\providecommand\BIBforeignlanguage[2]{{%
\expandafter\ifx\csname l@#1\endcsname\relax
\typeout{** WARNING: IEEEtran.bst: No hyphenation pattern has been}%
\typeout{** loaded for the language `#1'. Using the pattern for}%
\typeout{** the default language instead.}%
\else
\language=\csname l@#1\endcsname
\fi
#2}}

\bibitem{freundlich13icra}
C.~Freundlich, P.~Mordohai, and M.~M. Zavlanos, ``A hybrid control approach to
  the next-best-view problem using stereo vision,'' in \emph{{IEEE} Int. Conf.
  on Robotics and Automation ({ICRA})}, Karlsruhe, DE, June 2013, pp.
  4478--4483.

\bibitem{freundlich13cdc}
------, ``Hybrid control for mobile target localization with stereo vision,''
  in \emph{{IEEE} Conf. on Decision and Control ({CDC})}, Firenze, Italy,
  December 2013, pp. 2635--2640.

\bibitem{blostein87}
S.~D. Blostein and T.~S. Huang, ``Error analysis in stereo determination of 3-d
  point positions,'' \emph{{IEEE} Trans. on Pattern Analysis and Machine
  Intell.}, vol.~9, no.~6, pp. 752--766, 1987.

\bibitem{matthies87}
L.~H. Matthies and S.~A. Shafer, ``Error modelling in stereo navigation,''
  \emph{IEEE Journal of Robotics and Automation}, vol.~3, no.~3, pp. 239--250,
  1987.

\bibitem{chang94}
C.~C. Chang, S.~Chatterjee, and P.~R. Kube, ``A quantization error analysis for
  convergent stereo,'' in \emph{Int. Conf. on Image Proc. ({ICIP})}.\hskip 1em
  plus 0.5em minus 0.4em\relax Austin, Texas: IEEE, November 13--16 1994, pp.
  II: 735--739.

\bibitem{le1996optimization}
J.-P. Le~Cadre and H.~Gauvrit, ``Optimization of the observer motion for
  bearings-only target motion analysis,'' in \emph{Data Fusion Symposium, 1996.
  ADFS'96., First Australian}.\hskip 1em plus 0.5em minus 0.4em\relax IEEE,
  1996, pp. 190--195.

\bibitem{passerieux1998optimal}
J.-M. Passerieux and D.~Van~Cappel, ``Optimal observer maneuver for
  bearings-only tracking,'' \emph{{IEEE} Trans. on Aerospace and Elec. Sys.},
  vol.~34, no.~3, pp. 777--788, 1998.

\bibitem{logothetis1997information}
A.~Logothetis \emph{et~al.}, ``An information theoretic approach to observer
  path design for bearings-only tracking,'' in \emph{{IEEE} Conf. on Decision
  and Control ({CDC})}, vol.~4, San Diego, CA, December 1997, pp. 3132--3137.

\bibitem{stroupe05}
A.~W. Stroupe and T.~Balch, ``Value-based action selection for observation with
  robot teams using probabilistic techniques,'' \emph{Robotics and Autonomous
  Systems}, vol.~50, no. 2-3, pp. 85 -- 97, 2005.

\bibitem{zhou_roumeliotis11}
K.~Zhou and S.~Roumeliotis, ``Multirobot active target tracking with
  combinations of relative observations,'' \emph{{IEEE} Trans. on Robotics},
  vol.~27, no.~4, pp. 678 --695, 2011.

\bibitem{olfati2007distributed}
R.~Olfati-Saber, ``Distributed tracking for mobile sensor networks with
  information-driven mobility,'' in \emph{{IEEE} American Control Conf.
  ({ACC})}, New York, NY, July 2007, pp. 4606--4612.

\bibitem{chung06}
T.~H. Chung, J.~W. Burdick, and R.~M. Murray, ``A decentralized motion
  coordination strategy for dynamic target tracking,'' in \emph{{IEEE} Int.
  Conf. on Robotics and Automation ({ICRA})}.\hskip 1em plus 0.5em minus
  0.4em\relax Orlando, FL: IEEE, 2006, pp. 2416 --2422.

\bibitem{fox00}
D.~Fox, W.~Burgard, H.~Kruppa, and S.~Thrun, ``A probabilistic approach to
  collaborative multi-robot localization,'' \emph{Autonomous Robots}, vol.~8,
  no.~3, pp. 325--344, 2000.

\bibitem{roumeliotis02}
S.~Roumeliotis and G.~Bekey, ``Distributed multirobot localization,''
  \emph{{IEEE} Trans. on Robotics and Automation}, vol.~18, no.~5, pp. 781 --
  795, 2002.

\bibitem{spletzer03}
J.~R. Spletzer and C.~J. Taylor, ``Dynamic sensor planning and control for
  optimally tracking targets,'' \emph{Int. J. Robotics Research}, vol.~22,
  no.~1, pp. 7--20, 2003.

\bibitem{yang07}
P.~Yang, R.~Freeman, and K.~Lynch, ``Distributed cooperative active sensing
  using consensus filters,'' in \emph{{IEEE} Int. Conf. on Robotics and
  Automation ({ICRA})}.\hskip 1em plus 0.5em minus 0.4em\relax Roma, Italy:
  IEEE, April 2007, pp. 405 --410.

\bibitem{morbidi2013active}
F.~Morbidi and G.~L. Mariottini, ``Active target tracking and cooperative
  localization for teams of aerial vehicles,'' \emph{IEEE Transactions on
  Control Systems Technology}, vol.~21, no.~5, pp. 1694--1707, 2013.

\bibitem{ponda09}
S.~Ponda and E.~Frazzoli, ``Trajectory optimization for target localization
  using small unmanned aerial vehicles,'' in \emph{AIAA Conf. on Guidance,
  Navigation, and Control}.\hskip 1em plus 0.5em minus 0.4em\relax Chicago, IL:
  IEEE, August 2009.

\bibitem{Adurthi13}
N.~Adurthi \emph{et~al.}, ``Optimal information collection for nonlinear
  systems- an application to multiple target tracking and localization,'' in
  \emph{{IEEE} American Control Conf. ({ACC})}, 2013, pp. 3864--3869.

\bibitem{ding2012coordinated}
C.~Ding, A.~Morye, J.~Farrell, and A.~Roy-Chowdhury, ``Coordinated sensing and
  tracking for mobile camera platforms,'' in \emph{{IEEE} American Control
  Conf. ({ACC})}.\hskip 1em plus 0.5em minus 0.4em\relax Montreal, Canada:
  IEEE, June 2012, pp. 5114--5119.

\bibitem{bajcsy88}
R.~Bajcsy, ``Active perception,'' \emph{Proceedings of the IEEE}, vol.~76,
  no.~8, pp. 966--1005, 1988.

\bibitem{trummer10}
M.~Trummer \emph{et~al.}, ``Online next-best-view planning for accuracy
  optimization using an extended e-criterion,'' in \emph{Int. Conf. on Pattern
  Recognition}.\hskip 1em plus 0.5em minus 0.4em\relax Istanbul, Turkey: IEEE,
  August 2010, pp. 1642--1645.

\bibitem{Wenhardt06}
S.~Wenhardt \emph{et~al.}, ``An information theoretic approach for next best
  view planning in 3-d reconstruction,'' in \emph{{IEEE} Conf. on Computer
  Vision and Pattern Recognition ({CVPR})}, vol.~1, 2007, pp. 103 --106.

\bibitem{dunn_iros09}
E.~Dunn, J.~Van Den~Berg, and J.-M. Frahm, ``Developing visual sensing
  strategies through next best view planning,'' in \emph{{IEEE} Int. Conf. on
  Intell. Robots and Systems ({IROS})}.\hskip 1em plus 0.5em minus 0.4em\relax
  St. Louis, USA: IEEE, oct. 2009, pp. 4001 --4008.

\bibitem{Shade10}
R.~Shade and P.~Newman, ``Discovering and mapping complete surfaces with
  stereo,'' in \emph{{IEEE} Int. Conf. on Robotics and Automation
  ({ICRA})}.\hskip 1em plus 0.5em minus 0.4em\relax Anchorage, AK: IEEE, May
  2010, pp. 3910--3915.

\bibitem{galceran2013}
E.~Galceran and M.~Carreras, ``A survey on coverage path planning for
  robotics,'' \emph{Robotics and Autonomous Systems}, vol.~61, no.~12, pp.
  1258--1276, 2013.

\bibitem{wang2007}
P.~Wang, R.~Krishnamurti, and K.~Gupta, ``View planning problem with combined
  view and traveling cost,'' in \emph{{IEEE} Int. Conf. on Robotics and
  Automation ({ICRA})}.\hskip 1em plus 0.5em minus 0.4em\relax Roma, Italy:
  IEEE, April 2007, pp. 711--716.

\bibitem{papadopoulos2013}
G.~Papadopoulos, H.~Kurniawati, and N.~M. Patrikalakis, ``Asymptotically
  optimal inspection planning using systems with differential constraints,'' in
  \emph{{IEEE} Int. Conf. on Robotics and Automation ({ICRA})}.\hskip 1em plus
  0.5em minus 0.4em\relax Karlsruhe, DE: IEEE, June 2013, pp. 4126--4133.

\bibitem{hollinger13}
G.~A. Hollinger, B.~Englot, F.~S. Hover, U.~Mitra, and G.~S. Sukhatme, ``Active
  planning for underwater inspection and the benefit of adaptivity,''
  \emph{Int. J. Robotics Research}, vol.~32, no.~1, pp. 3--18, 2013.

\bibitem{singh2007simulation}
S.~S. Singh, N.~Kantas, B.-N. Vo, A.~Doucet, and R.~J. Evans,
  ``Simulation-based optimal sensor scheduling with application to observer
  trajectory planning,'' \emph{Automatica}, vol.~43, no.~5, pp. 817--830, 2007.

\bibitem{logothetis1998comparison}
A.~Logothetis, A.~Isaksson, and R.~J. Evans, ``Comparison of suboptimal
  strategies for optimal own-ship maneuvers in bearings-only tracking,'' in
  \emph{American Control Conference, 1998. Proceedings of the 1998},
  vol.~6.\hskip 1em plus 0.5em minus 0.4em\relax Philadelphia, PA: IEEE, June
  1998, pp. 3334--3338.

\bibitem{frew2003trajectory}
E.~W. Frew, ``Trajectory design for target motion estimation using monocular
  vision,'' Ph.D. dissertation, Ph. D. dissertation, Stanford University,
  Stanford, CA, 2003.

\bibitem{foerstner05}
W.~F{\"o}rstner, ``Uncertainty and projective geometry,'' in \emph{Handbook of
  Geometric Computing}, E.~Bayro-Corrochano, Ed.\hskip 1em plus 0.5em minus
  0.4em\relax Springer, 2005, pp. 493--534.

\bibitem{Singer70}
R.~Singer, ``Estimating optimal tracking filter performance for manned
  maneuvering targets,'' \emph{{IEEE} Trans. on Aerospace and Elec. Sys.}, vol.
  AES-6, no.~4, pp. 473 --483, july 1970.

\bibitem{Li03_part1}
X.~Rong~Li and V.~Jilkov, ``Survey of maneuvering target tracking. part {I}.
  dynamic models,'' \emph{{IEEE} Trans. on Aerospace and Elec. Sys.}, vol.~39,
  no.~4, pp. 1333 -- 1364, oct. 2003.

\bibitem{welch01}
G.~Bishop and G.~Welch, ``An introduction to the kalman filter,'' \emph{Proc.
  of SIGGRAPH, Course}, vol.~8, no. 27599-23175, p.~41, 2001.

\bibitem{scharstein2002}
D.~Scharstein and R.~Szeliski, ``A taxonomy and evaluation of dense two-frame
  stereo correspondence algorithms,'' \emph{Int. J. Computer Vision}, vol.~47,
  no. 1-3, pp. 7--42, 2002.

\bibitem{wenhardt07}
S.~Wenhardt, B.~Deutsch, E.~Angelopoulou, and H.~Niemann, ``Active visual
  object reconstruction using d-, e-, and t-optimal next best views,'' in
  \emph{{IEEE} Conf. on Computer Vision and Pattern Recognition ({CVPR})},
  Minneapolis, MN, June 2007.

\bibitem{Zavlanos2008}
M.~M. Zavlanos and G.~J. Pappas, ``A dynamical systems approach to weighted
  graph matching,'' \emph{Automatica}, vol.~44, no.~11, pp. 2817--2824, Nov.
  2008.

\bibitem{bouguet2004camera}
J.~Bouguet, ``Camera calibration toolbox for matlab (software),''
  \emph{California Institute of Technology, Pasadena, http://www. vision.
  caltech. edu/bouguetj/calib\_doc}, 2004.

\bibitem{freundlich15cvpr}
C.~Freundlich, P.~Mordohai, and M.~M. Zavlanos, ``Exact bias correction and
  covariance estimation for stereo vision,'' in \emph{{IEEE} Conf. on Computer
  Vision and Pattern Recognition ({CVPR})}, Boston, MA, June 2015, pp.
  3296--3304.

\end{thebibliography}

\end{document}